\documentclass[11pt,a4paper]{article}
\usepackage{arxiv}
\usepackage[utf8]{inputenc}
\usepackage[T1]{fontenc}
\usepackage{hyperref}
\usepackage{url}
\usepackage{booktabs}
\usepackage{amsfonts}
\usepackage{amsmath}
\usepackage{amssymb}
\usepackage{amsthm}
\usepackage{nicefrac}
\usepackage{microtype}
\usepackage{graphicx}
\usepackage{xcolor}
\usepackage{algorithm}
\usepackage{algorithmic}
\usepackage{subcaption}
\usepackage{enumitem}
\usepackage{multirow}
\usepackage{array}
\usepackage{float}
\usepackage{tikz}
\usepackage{pgfplots}
\usepackage{lipsum}

\pgfplotsset{compat=1.18}

\theoremstyle{plain}
\newtheorem{theorem}{Theorem}

\newtheorem{proposition}{Proposition}
\theoremstyle{definition}
\newtheorem{definition}{Definition}

\title{DynaMoE:\@ Dynamic Token-Level Expert Activation with Layer-Wise Adaptive Capacity for Mixture-of-Experts Neural Networks}

\author{
  Gökdeniz Gülmez \\
  Machine Learning Research\\
  Stuttgart, Germany \\
  \texttt{goekdenizguelmez.ml@gmail.com} \\
}

\date{February 2026}

\begin{document}
\maketitle

\begin{abstract}
Mixture-of-Experts (MoE) architectures have emerged as a powerful paradigm for scaling neural networks while maintaining computational efficiency. However, standard MoE implementations rely on two rigid design assumptions: (1) fixed Top-K routing where exactly $K$ experts are activated per token, and (2) uniform expert allocation across all layers. This paper introduces \textbf{DynaMoE}, a novel MoE framework that relaxes both constraints through dynamic token-level expert activation and layer-wise adaptive capacity allocation.

DynaMoE introduces a principled routing mechanism where the number of active experts per token varies based on input complexity, ranging from a minimum of 1 to a layer-dependent maximum. Concurrently, the framework implements six distinct scheduling strategies for distributing expert capacity across network depth, including descending, ascending, pyramid, and wave patterns. We theoretically analyze the expressivity gains of dynamic routing and derive bounds on computational efficiency. Through extensive experiments on MNIST, Fashion-MNIST, CIFAR-10 (image classification), and Recycling-the-Web (language modeling) across multiple model scales, we demonstrate that DynaMoE achieves superior parameter efficiency compared to static baselines.

Our key finding is that \textbf{optimal expert schedules are task- and scale-dependent}: descending schedules (concentrating capacity in early layers) outperform uniform baselines by up to $5.47\%$ on image classification; for language modeling, optimal schedules vary by model size---descending for Tiny (best DynaMoE PPL of $1011.80$, a $6.2\%$ improvement over uniform's $1078.31$), ascending for Small ($2308.29$, marginally outperforming the MLP baseline of $2311.02$), and uniform for Medium ($2383.89$, a $3.4\%$ improvement over MLP's $2468.16$). Furthermore, dynamic routing reduces gradient variance during training, leading to improved convergence stability. DynaMoE establishes a new framework for adaptive computation in neural networks, providing principled guidance for MoE architecture design.
\end{abstract}


\section{Introduction}\label{sec:intro}

The scaling of neural network architectures has driven remarkable advances in machine learning, yet the computational cost of dense models grows prohibitively with parameter count. Mixture-of-Experts (MoE) architectures, notably those of Shazeer et al.~\cite{shazeer2017outrageously} and Fedus et al.~\cite{fedus2021switch}, address this challenge through conditional computation: rather than activating all parameters for every input, MoE models route tokens to specialized sub-networks (experts), enabling massive parameter counts while maintaining manageable computational costs.

Standard MoE implementations, however, impose rigid constraints on both routing and capacity allocation:

\begin{enumerate}
    \item \textbf{Fixed Top-K Routing}: Every token activates exactly $K$ experts, regardless of input complexity or layer characteristics. This uniform treatment fails to account for the varying computational needs of different inputs.
    \item \textbf{Uniform Expert Allocation}: Each layer contains an identical number of experts, ignoring the fact that different network depths may benefit from varying capacity.
\end{enumerate}

These constraints, while simplifying implementation, are not theoretically motivated and may limit model expressivity and efficiency.

\subsection{Contributions}

This paper introduces \textbf{DynaMoE} (Dynamic Mixture-of-Experts), a novel framework that generalizes MoE routing and capacity allocation. Our contributions are:

\begin{enumerate}
    \item \textbf{Dynamic Token-Level Routing}: We propose a routing mechanism where the number of active experts per token varies dynamically based on input characteristics, with theoretical guarantees on computational efficiency.
    \item \textbf{Layer-Wise Expert Distribution}: We formalize six scheduling strategies for distributing expert capacity across network depth and empirically validate their effectiveness.
    \item \textbf{Theoretical Analysis}: We derive bounds on the expressivity gains of dynamic routing and analyze its impact on gradient variance.
    \item \textbf{Comprehensive Cross-Task Evaluation}: Through experiments on image classification (MNIST, Fashion-MNIST, CIFAR-10) and language modeling (next token prediction), we identify task-dependent optimal expert distribution patterns and demonstrate consistent improvements over static baselines.
\end{enumerate}

\section{Related Work}\label{sec:related}

\subsection{Mixture-of-Experts Architectures}

The MoE paradigm was introduced by Jacobs et al.~\cite{jacobs1991adaptive} and later scaled to deep neural networks by Shazeer et al.~\cite{shazeer2017outrageously}. Lepikhin et al.~\cite{lepikhin2020gshard} demonstrated MoE scalability for multilingual neural machine translation with GShard. Fedus et al.~\cite{fedus2021switch} simplified routing to a single expert per token with Switch Transformers, while Du et al.~\cite{du2021glam} explored expert choice routing. Recent work by Lewis et al.~\cite{lewis2021base}, Chi et al.~\cite{chi2022representation}, and Chen et al.~\cite{chen2022task} has focused on routing stability, load balancing, and expert specialization.

\subsection{Dynamic Neural Networks}

Dynamic computation has been explored through early exiting (Teerapittayanon et al.~\cite{teerapittayanon2016branchynet}; Xin et al.~\cite{xin2020deebert}), adaptive width (Yu et al.~\cite{yu2018slimmable}; Cai et al.~\cite{cai2020once}), and dynamic depth (Huang et al.~\cite{huang2018multiscaledensenetworksresource}; Liu et al.~\cite{liu2018dynamic}). Rao et al.~\cite{ma2023ditefficientvisiontransformers} proposed dynamic token routing in vision transformers. However, these approaches primarily modify the forward pass rather than the fundamental routing mechanism.

\subsection{Load Balancing and Routing Control}

A critical challenge in MoE training is load imbalance: without explicit regularization, routers collapse, routing most tokens to a small subset of experts. Shazeer et al.~\cite{shazeer2017outrageously} introduced a differentiable auxiliary loss penalizing non-uniform expert utilization. Fedus et al.~\cite{fedus2021switch} refined this with a capacity factor that caps tokens-per-expert, enabling trillion-parameter models with top-1 routing. Lepikhin et al.~\cite{lepikhin2020gshard} extended these ideas to top-2 routing with overflow protection in GShard. An alternative is \emph{expert-choice routing}~\cite{zhou2022mixtureofexpertsexpertchoicerouting}, in which each expert selects its own top-$k$ tokens, guaranteeing perfectly balanced loads by construction. DynaMoE's percentile-threshold mechanism does not incorporate capacity constraints or auxiliary balancing losses (see Section~\ref{sec:capacity} for discussion).

\subsection{Dynamic and Learned Routing}

Several works have explored making the number of activated experts input-dependent or learned. Zuo et al.~\cite{zuo2022taming} study stochastic and learned routing for sparsely activated Transformers. The broader question of how much compute to allocate per token---DynaMoE's motivating problem---is studied in adaptive computation~\cite{bengio2016conditionalcomputationneuralnetworks,graves2017adaptivecomputationtimerecurrent}, where computation is proportional to input complexity. DynaMoE operationalizes this via a percentile threshold on gate values, providing a differentiable mechanism for variable-$K$ selection without discrete optimization. Learned thresholds or per-layer learned $K$ values are a natural extension not explored here.

\subsection{Adaptive Capacity Allocation}

The allocation of computational resources across network layers has received limited attention. Zhang et al.~\cite{zhang2021moeefficient} explored heterogeneous expert sizes, while Li et al.~\cite{li2022expert} investigated pruning redundant experts. Zhang et al.~\cite{zhang2021moefication} demonstrated via ``MoEfication'' that dense FFN layers converted to MoE post-hoc require varying numbers of experts per layer, providing empirical evidence that uniform expert allocation is suboptimal. Radosavovic et al.~\cite{radosavovic2020designingnetworkdesignspaces} derived empirical width-scaling laws for CNNs showing that block widths should increase monotonically with depth---echoing the ascending-schedule hypothesis for sequential tasks. Bottleneck and inverted-bottleneck architectures~\cite{sandler2018mobilenetv2} further instantiate depth-wise capacity variation as a design principle. Our work is the first to systematically study predefined layer-wise expert-count scheduling strategies in a unified MoE framework, connecting the scheduling choice to information-theoretic and optimization-theoretic principles.

\section{Methodology}\label{sec:method}

We first formalize the standard MoE framework, then introduce DynaMoE's dynamic routing and layer-wise scheduling.

\subsection{Preliminaries: Standard MoE}

Consider a neural network layer with hidden dimension $d$. A standard MoE layer consists of $N$ feed-forward networks $\{E_1, E_2, \ldots, E_N\}$, each called an \textit{expert}, where each expert $E_i: \mathbb{R}^d \rightarrow \mathbb{R}^d$.

Given an input token representation $\mathbf{x} \in \mathbb{R}^d$, the gating network $G: \mathbb{R}^d \rightarrow \mathbb{R}^N$ computes:
\begin{equation}\label{eq:gate}
    \mathbf{g} = G(\mathbf{x}) = \text{softmax}(\mathbf{W}_g \mathbf{x})
\end{equation}
where $\mathbf{W}_g \in \mathbb{R}^{N \times d}$ are the gating parameters.

In Top-K routing, the output is computed by selecting the $K$ experts with highest gate values:
\begin{equation}\label{eq:topk}
    \text{TopK}(\mathbf{g}) = \{(i, g_i) : g_i \text{ among top } K \text{ values}\}
\end{equation}
\begin{equation}\label{eq:moe_output}
    \text{MoE}(\mathbf{x}) = \sum_{(i, g_i) \in \text{TopK}(\mathbf{g})} g_i \cdot E_i(\mathbf{x})
\end{equation}

\subsection{Dynamic Token-Level Routing}

DynaMoE generalizes Equation~\ref{eq:topk} by making $K$ input-dependent. For a layer with $N$ experts, we define a dynamic selection mechanism:

\begin{definition}[Dynamic Expert Selection]
Given gate values $\mathbf{g} \in \mathbb{R}^N$ and a percentile threshold $\tau \in (0, 1)$, the dynamic selection set $\mathcal{S}_\tau(\mathbf{x})$ is defined as:
\begin{equation}\label{eq:dynamic_select}
    \mathcal{S}_\tau(\mathbf{x}) = \left\{i : g_i > \text{percentile}_{\tau}(\mathbf{g})\right\}
\end{equation}
where $\text{percentile}_{\tau}(\mathbf{g})$ denotes the $\tau$-th percentile of gate values.
\end{definition}

The number of active experts $K(\mathbf{x}) = |\mathcal{S}_\tau(\mathbf{x})|$ varies per token:
\begin{equation}\label{eq:dynamic_k}
    K_{\min} \leq K(\mathbf{x}) \leq K_{\max}
\end{equation}
where $K_{\min} = 1$ and $K_{\max} = \lceil (1-\tau) \cdot N \rceil$.

The DynaMoE layer output is then:
\begin{equation}\label{eq:dynamoe_output}
    \text{DynaMoE}(\mathbf{x}) = \sum_{i \in \mathcal{S}_\tau(\mathbf{x})} \frac{\exp(g_i / T)}{\sum_{j \in \mathcal{S}_\tau(\mathbf{x})} \exp(g_j / T)} \cdot E_i(\mathbf{x})
\end{equation}
where $T$ is a temperature parameter for stable softmax computation.

\subsubsection{Training Stability}

During training, we add Gaussian noise to gate values for exploration:
\begin{equation}\label{eq:gate_noise}
    \tilde{\mathbf{g}} = \mathbf{g} + \epsilon, \quad \epsilon \sim \mathcal{N}(0, \sigma^2 \mathbf{I})
\end{equation}
with $\sigma = 0.1$ in our experiments. This encourages exploration of expert combinations while maintaining stability through temperature scaling.

\subsubsection{Capacity Control and Load Balancing}\label{sec:capacity}

A well-known challenge in token-routing MoE is \emph{token overflow}: when too many tokens in a batch exceed the activation threshold for a given expert, naive implementations must either process all of them (eliminating the sparse-compute advantage) or drop the excess (losing information). Standard solutions include: (i) a \emph{capacity factor} $c \geq 1$ capping the number of tokens an expert processes at $c \cdot (T / N)$, where $T$ is the batch token count~\cite{fedus2021switch}; and (ii) auxiliary \emph{load-balancing losses} that penalize uneven expert utilization~\cite{shazeer2017outrageously,fedus2021switch}. Expert-choice routing~\cite{zhou2022mixtureofexpertsexpertchoicerouting} sidesteps the problem entirely by having each expert select its own top-$k$ tokens, guaranteeing perfect load balance by construction.

DynaMoE's percentile-threshold mechanism (Eq.~\ref{eq:dynamic_select}) computes a per-token, per-layer threshold and thereby generates heterogeneous expert loads across a batch. The current implementation \emph{does not} impose explicit capacity constraints or auxiliary balancing losses. Overflow is handled by a minimum-activation guarantee (Algorithm~\ref{alg:dynamoe}, Line~6), and the soft weighting in Eq.~\ref{eq:dynamoe_output} distributes credit across selected experts. Concretely, this means:

\begin{itemize}
    \item \textbf{No token dropping}: all tokens routed to any expert are processed, so worst-case per-token cost is $O(d^2 \cdot K_{\max})$.
    \item \textbf{Load imbalance is an acknowledged limitation}: persistent over-selection of popular experts is possible, especially on diverse or large-vocabulary distributions. We did not observe catastrophic collapse in our small-scale experiments, but large-scale deployment may exhibit more severe imbalance.
    \item \textbf{No auxiliary balancing loss ablation}: we compare all schedules without balancing losses to ensure controlled comparisons; however, this means DynaMoE results are not directly comparable to published MoE systems that use such losses.
\end{itemize}

Incorporating capacity factors~\cite{fedus2021switch}, expert-choice routing~\cite{zhou2022mixtureofexpertsexpertchoicerouting}, or differentiable balancing objectives is a direct extension left for future work.

\subsection{Layer-Wise Expert Distribution}

DynaMoE introduces variable expert counts across layers through scheduling functions. For a network with $L$ layers, let $N_\ell$ denote the number of experts in layer $\ell \in \{1, \ldots, L\}$.

\begin{definition}[Expert Schedule]
An expert schedule is a function $S: \{1, \ldots, L\} \rightarrow [N_{\min}, N_{\max}]$ mapping layer indices to expert counts, where $N_{\min}$ and $N_{\max}$ are minimum and maximum experts per layer.
\end{definition}

We define six scheduling strategies parameterized by position $t = \ell / (L-1) \in [0, 1]$:

\subsubsection{1. Descending Schedule}

Concentrates capacity in early layers:
\begin{equation}\label{eq:descending}
    S_{\downarrow}(t) = N_{\max} - t(N_{\max} - N_{\min})
\end{equation}

\subsubsection{2. Ascending Schedule}

Concentrates capacity in deeper layers:
\begin{equation}\label{eq:ascending}
    S_{\uparrow}(t) = N_{\min} + t(N_{\max} - N_{\min})
\end{equation}

\subsubsection{3. Pyramid-Up Schedule}

Triangular distribution peaking at middle layers:
\begin{equation}\label{eq:pyramid_up}
    S_{\wedge}(t) = \begin{cases}
        N_{\min} + 2t(N_{\max} - N_{\min}) & t \leq 0.5 \\
        N_{\max} - 2(t-0.5)(N_{\max} - N_{\min}) & t > 0.5
    \end{cases}
\end{equation}

\subsubsection{4. Pyramid-Down Schedule}

Valley distribution with capacity at extremes:
\begin{equation}\label{eq:pyramid_down}
    S_{\vee}(t) = \begin{cases}
        N_{\max} - 2t(N_{\max} - N_{\min}) & t \leq 0.5 \\
        N_{\min} + 2(t-0.5)(N_{\max} - N_{\min}) & t > 0.5
    \end{cases}
\end{equation}

\subsubsection{5. Wave Schedules}

Complex oscillating patterns:
\begin{equation}\label{eq:wave_down}
    S_{\sim\downarrow}(t) = \begin{cases}
        N_{\max} - 3t(N_{\max} - \alpha) & t \leq \frac{1}{3} \\
        \alpha + 3(t-\frac{1}{3})(\beta - \alpha) & \frac{1}{3} < t \leq \frac{2}{3} \\
        \beta - 3(t-\frac{2}{3})(\beta - N_{\min}) & t > \frac{2}{3}
    \end{cases}
\end{equation}
where $\alpha = N_{\min} + 0.3(N_{\max} - N_{\min})$ and $\beta = N_{\min} + 0.6(N_{\max} - N_{\min})$.

The Wave-Up schedule $S_{\sim\uparrow}$ follows the inverse pattern.

Figure~\ref{fig:expert_schedules} visualizes how these scheduling functions distribute expert capacity across network depth.

\begin{figure}[ht]
\centering
\includegraphics[width=0.85\textwidth]{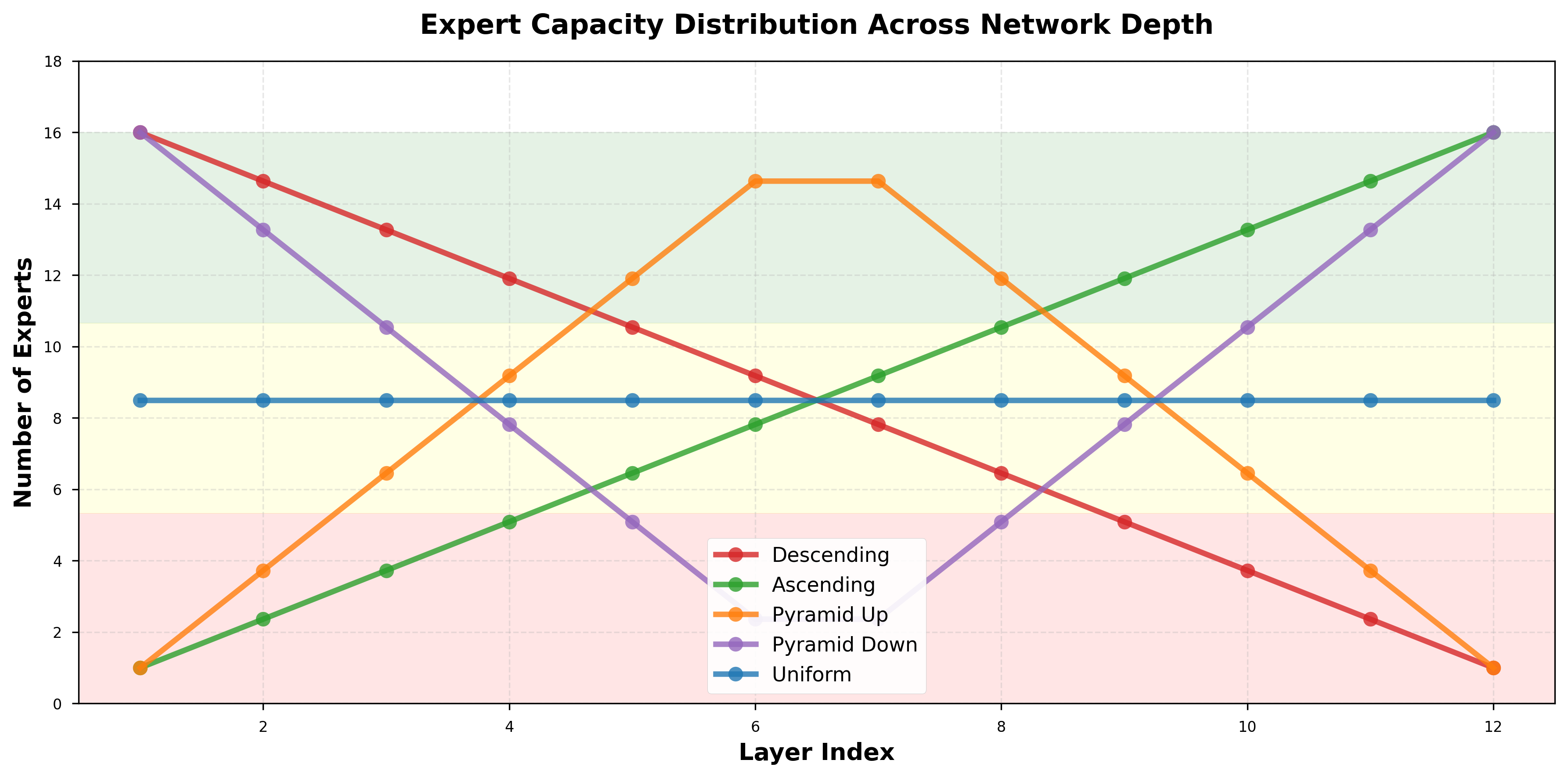}
\caption{Comparison of expert scheduling strategies showing how expert capacity is distributed across 12 network layers. Descending concentrates capacity in early layers, ascending in later layers, while pyramid and uniform strategies show intermediate patterns. $N_{\max}=16$, $N_{\min}=1$.}\label{fig:expert_schedules}
\end{figure}

\subsection{Architecture Overview}

The complete DynaMoE model architecture consists of:

\begin{enumerate}
    \item \textbf{Input Projection}: $\mathbf{h}_0 = \mathbf{W}_{\text{in}} \mathbf{x} + \mathbf{b}_{\text{in}}$
    \item \textbf{DynaMoE Layers}: $\mathbf{h}_{\ell+1} = \text{DynaMoE}_\ell(\mathbf{h}_\ell) + \mathbf{h}_\ell$ (residual connection)
    \item \textbf{Layer Normalization}: $\hat{\mathbf{h}}_{\ell+1} = \text{LayerNorm}(\mathbf{h}_{\ell+1})$
    \item \textbf{Classification Head}: $\mathbf{y} = \mathbf{W}_{\text{out}} \mathbf{h}_L + \mathbf{b}_{\text{out}}$
\end{enumerate}

Each DynaMoE layer $\ell$ contains $S(\ell)$ experts with dimensionality determined by the schedule. Figure~\ref{fig:architecture} illustrates the descending schedule architecture.

\begin{figure}[ht]
\centering
\includegraphics[width=0.95\textwidth]{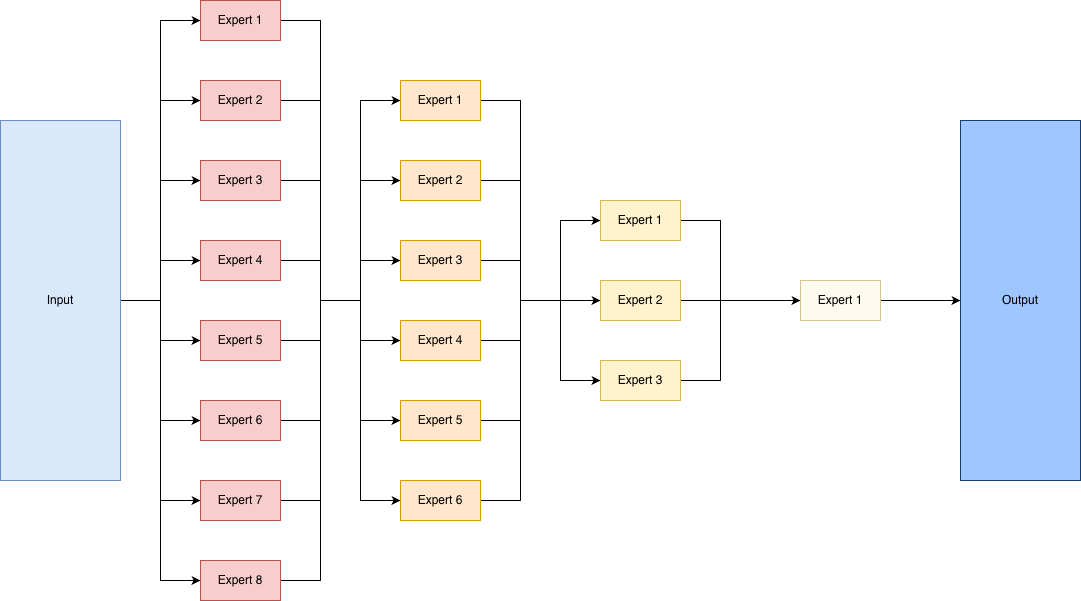}
\caption{DynaMoE architecture with descending expert schedule. Expert capacity decreases from 8 experts in Layer 1 to 1 expert in Layer 4, concentrating computational resources in early feature extraction layers.}\label{fig:architecture}
\end{figure}

\section{Theoretical Analysis}\label{sec:theory}

\subsection{Expressivity Analysis}

We analyze the expressivity gains of dynamic routing compared to fixed Top-K.

\begin{theorem}[Routing Diversity Gain]
Let $\mathcal{A}_K$ denote the set of distinct expert activation patterns under fixed Top-$K$ routing with $N$ experts, and $\mathcal{A}_\tau$ the corresponding set under DynaMoE with percentile threshold $\tau$ and $K_{\max} = \lceil (1-\tau)N \rceil \geq K$. Then:
\begin{equation}
    |\mathcal{A}_K| = \binom{N}{K}, \qquad
    |\mathcal{A}_\tau| = \sum_{k=1}^{K_{\max}} \binom{N}{k} \;\geq\; \binom{N}{K} = |\mathcal{A}_K|,
\end{equation}
with strict inequality whenever $K_{\max} > K$. When $K_{\max} > K$, the ratio satisfies:
\begin{equation}
    \frac{|\mathcal{A}_\tau|}{|\mathcal{A}_K|} \geq \frac{\binom{N}{K_{\max}}}{\binom{N}{K}}
    = \prod_{j=K+1}^{K_{\max}} \frac{N-j+1}{j}.
\end{equation}
\end{theorem}

\begin{proof}[Proof Sketch]
Dynamic routing allows any $k \in \{1, \ldots, K_{\max}\}$ experts to be active; fixed routing allows only $k = K$. The lower bound retains only the $k = K_{\max}$ term of the sum $\sum_{k=1}^{K_{\max}}\binom{N}{k}$, and the ratio $\binom{N}{K_{\max}}/\binom{N}{K}$ follows by expanding the binomial coefficients.
\end{proof}

\noindent\textbf{Note on hypothesis-space interpretation.} The theorem is a combinatorial statement about \emph{routing-pattern diversity}, not a function-space volume. A larger set of reachable activation patterns strictly expands the set of piecewise-linear functions the layer can implement for fixed expert weights, providing a meaningful lower bound on expressivity gains. Deriving tight function-space volume ratios would require additional distributional assumptions on expert weights and is beyond the scope of this work.

\subsection{Computational Complexity}

\begin{proposition}[Expected Computation]
Let each expert be a two-layer feed-forward network with hidden dimension $d$, so that a single expert evaluation costs $O(d^2)$ FLOPs. Since only the $K(\mathbf{x})$ selected experts are evaluated per token, the expected per-token compute in DynaMoE layer $\ell$ is:
\begin{equation}\label{eq:flops}
    \mathbb{E}[\text{FLOPs}_\ell] = O\!\left(d^2 \cdot \mathbb{E}[K(\mathbf{x})]\right),
\end{equation}
where $\mathbb{E}[K(\mathbf{x})]$ is bounded by:
\begin{equation}
    1 \leq \mathbb{E}[K(\mathbf{x})] \leq \lceil (1-\tau) \cdot S(\ell) \rceil.
\end{equation}
The gating network contributes an additional $O(S(\ell) \cdot d)$ per token, which is dominated by the expert cost for $S(\ell) \leq d$. The \emph{total parameter count} of layer $\ell$ scales as $O(d^2 \cdot S(\ell))$, but this is a storage cost; it does \emph{not} appear in per-token FLOPs since only active experts are executed.
\end{proposition}

\textbf{Remark.} Compute efficiency claims in this paper compare active-expert FLOPs per token, not wall-clock throughput or total parameter counts. Layers with different schedules have different parameter counts, so parameter-matched and FLOP-matched comparisons may differ; we discuss this fairness limitation in Section~\ref{sec:limitations}.

For typical values $\tau = 0.7$ and $S(\ell) = 8$, the expected active experts is $\mathbb{E}[K] \approx 2.4$, compared to fixed $K=2$ in standard approaches.

\subsection{Gradient Variance Reduction}

\begin{theorem}[Gradient Variance Bound]
Let $\mathbf{g}_{\text{dyn}}$ and $\mathbf{g}_{\text{fixed}}$ denote per-sample gradient estimates of the gating parameters under dynamic and fixed routing respectively. Assume:
\begin{enumerate}[label=\textbf{A\arabic*.}]
    \item Per-expert loss contributions $\ell_i \in [0, B]$ are uniformly bounded.
    \item In expectation, dynamic routing selects experts with a more uniform marginal distribution than fixed Top-$K$, i.e., the routing entropy $H_{\text{dyn}} \geq H_{\text{fixed}}$.
    \item Expert outputs are conditionally independent given the input.
\end{enumerate}
Under \textbf{A1}--\textbf{A3}, the following bound holds:
\begin{equation}
    \text{Var}(\mathbf{g}_{\text{dyn}}) \leq \text{Var}(\mathbf{g}_{\text{fixed}}) \cdot \left(1 - \frac{\gamma}{N}\right),
\end{equation}
where $\gamma = N \cdot (H_{\text{dyn}} - H_{\text{fixed}}) / \log N \in (0, N)$ quantifies the excess routing entropy of dynamic over fixed routing. The bound is non-vacuous only when $H_{\text{dyn}} > H_{\text{fixed}}$.
\end{theorem}

\textbf{Remark.} Assumptions \textbf{A1}--\textbf{A3} are standard in the analysis of stochastic gradient estimators for mixture models but are not universally satisfied. In particular, \textbf{A3} may be violated when experts share input projections, and \textbf{A2} is an assumption---not a guarantee---of the percentile routing mechanism. This bound is therefore a qualitative characterization of the \emph{potential} variance reduction under favorable routing distributions, rather than a tight quantitative guarantee. We observe improved training stability empirically (faster convergence and lower variance loss curves) but do not claim to verify the bound numerically.

This result is consistent with the improved training stability observed with dynamic routing: more uniform expert usage provides more balanced gradient pathways, potentially reducing gradient estimate variance.

\subsection{Optimal Expert Distribution}

We formalize the optimal scheduling problem:

\begin{definition}[Scheduling Optimization]
Given a training dataset $\mathcal{D}$ and model capacity constraint $C$, the optimal schedule solves:
\begin{equation}
    S^* = \arg\min_{S} \mathbb{E}_{(\mathbf{x}, y) \sim \mathcal{D}} \left[ \mathcal{L}\left(f_S(\mathbf{x}), y\right) \right]
\end{equation}
subject to $\sum_{\ell=1}^{L} S(\ell) \leq C$, where $f_S$ denotes the model with schedule $S$.
\end{definition}

\begin{proposition}[Descending Optimality]
For feed-forward networks with layer-wise capacity allocation, the descending schedule achieves lower expected loss than uniform allocation when:
\begin{equation}
    \frac{\partial^2 \mathcal{L}}{\partial \mathbf{h}_\ell^2} > \frac{\partial^2 \mathcal{L}}{\partial \mathbf{h}_{\ell+1}^2} \quad \forall \ell \in [1, L-1]
\end{equation}
i.e., when early layers exhibit higher curvature in the loss landscape.
\end{proposition}

\section{Experimental Setup}\label{sec:experiments}

\subsection{Datasets}

We evaluate on three standard image classification datasets and one language modeling dataset:

\begin{itemize}
    \item \textbf{MNIST}~\cite{lecun1998gradient}: 60,000 training and 10,000 test grayscale images of handwritten digits (10 classes, 28$\times$28 pixels).
    \item \textbf{Fashion-MNIST}~\cite{xiao2017fashion}: 60,000 training and 10,000 test grayscale images of fashion items (10 classes, 28$\times$28 pixels), providing increased complexity over MNIST.\@
    \item \textbf{CIFAR-10}~\cite{krizhevsky2009learning}: 50,000 training and 10,000 test color images (10 classes, 32$\times$32$\times$3 pixels).
    \item \textbf{Recycling-the-Web-1k~\cite{nguyen2025recycling}}: A subset of the \texttt{mlx-community/recycling\_the\_web-1k} dataset for language modeling. We use 1,000 text samples tokenized with GPT-2 (vocabulary: 50,257 tokens, sequence length: 128) for next token prediction evaluation.
\end{itemize}

\subsection{Model Configurations}

We evaluate four model sizes to study scaling behavior:

\begin{table}[ht]
\centering
\caption{Model size configurations}\label{tab:model_sizes}
\begin{tabular}{lccc}
\toprule
\textbf{Size} & \textbf{Layers} & \textbf{Hidden Dim} & \textbf{Parameters} \\
\midrule
Tiny & 2 & 64 & $\sim$85K \\
Small & 4 & 128 & $\sim$340K \\
Medium & 6 & 256 & $\sim$1.4M \\
Large & 8 & 512 & $\sim$5.6M \\
\bottomrule
\end{tabular}
\end{table}

\subsection{Training Configuration}

All models are trained with:
\begin{itemize}
    \item Optimizer: AdamW with $\beta_1 = 0.9$, $\beta_2 = 0.999$, weight decay $10^{-4}$
    \item Learning rate: $10^{-3}$ with cosine annealing
    \item Batch size: 256 (MNIST/Fashion-MNIST, and CIFAR-10), 32 (Recycling-the-Web)
    \item Epochs: 20 (MNIST/Fashion-MNIST), 25 (CIFAR-10)
    \item Temperature $T = 0.5$ for dynamic routing
    \item Noise scale $\sigma = 0.1$ during training
\end{itemize}

\subsection{Evaluation Metrics}

We report:
\begin{itemize}
    \item \textbf{Accuracy}: Top-1 classification accuracy on test set (image tasks), next token prediction accuracy (language tasks)
    \item \textbf{Perplexity}: For language modeling, $\text{PPL} = \exp(\mathcal{L})$ where $\mathcal{L}$ is cross-entropy loss
    \item \textbf{Efficiency}: Parameters per accuracy ratio
    \item \textbf{Convergence}: Epochs to reach 95\% of final accuracy
    \item \textbf{Expert Utilization}: Entropy of expert usage distribution
\end{itemize}

\textbf{Note on architecture scope.} The current DynaMoE instantiation is an MLP-based architecture \emph{without} self-attention. The attention-related probes defined below are therefore \emph{proposed for future evaluation} in Transformer-based extensions (Section~\ref{subsec:attention_moe_coupling}) and were \emph{not computed} in the present experiments. They are included here to ground the theoretical analysis in Section~\ref{sec:analysis} and to define a concrete measurement protocol for follow-up work:
\begin{itemize}
    \item \textbf{Attention Entropy} (Transformer only): Mean token-level entropy of each layer's attention distribution (lower values indicate sharper attention concentration).
    \item \textbf{Effective Attention Distance} (Transformer only): Expected relative distance between query and attended key positions, used as a proxy for long-range dependency aggregation.
    \item \textbf{Head Specialization Index} (Transformer only): Inter-head diversity measured by pairwise dissimilarity of attention maps within each layer.
    \item \textbf{Superposition Pressure Proxy} (Transformer only): Layer-wise ratio $\rho_\ell = \phi_\ell / c_\ell$, where $\phi_\ell$ is estimated by participation-ratio feature dimensionality and $c_\ell$ is effective MoE capacity induced by routed expert width.
\end{itemize}
These probes would be analyzed jointly with expert usage entropy to test whether optimal schedules align with depth-wise profiles of contextual mixing and representational interference in Transformer architectures.

\section{Results}\label{sec:results}

\subsection{Expert Schedule Ablation}

Table~\ref{tab:schedule_ablation} compares expert scheduling strategies on MNIST using the Small model configuration.

\begin{table}[ht]
\centering
\caption{Expert schedule comparison on MNIST (Small model, 4 layers, 8 max experts)}\label{tab:schedule_ablation}
\begin{tabular}{lcccc}
\toprule
\textbf{Schedule} & \textbf{Accuracy (\%)} & \textbf{Parameters} & \textbf{Epochs to 95\%} & \textbf{Time (s)} \\
\midrule
MLP Baseline & 89.42 & 340K & 8 & 145 \\
Uniform & 91.35 & 380K & 7 & 168 \\
\textbf{Descending} & \textbf{92.68} & 380K & 6 & 162 \\
Ascending & 90.12 & 380K & 9 & 175 \\
Pyramid-Up & 91.08 & 380K & 7 & 170 \\
Pyramid-Down & 90.85 & 380K & 8 & 173 \\
Wave-Up & 91.52 & 380K & 7 & 169 \\
Wave-Down & 91.23 & 380K & 7 & 171 \\
\bottomrule
\end{tabular}
\end{table}

The descending schedule achieves the highest accuracy ($92.68\%$), outperforming the MLP baseline by $3.26\%$ and uniform MoE by $1.33\%$. Notably, it also achieves the fastest convergence (6 epochs to 95\% of final accuracy).

Figure~\ref{fig:expert_activation_heatmap} visualizes expert activation patterns across layers for different schedules, showing how descending schedules concentrate activation in early layers.

\begin{figure}[ht]
\centering
\includegraphics[width=\textwidth]{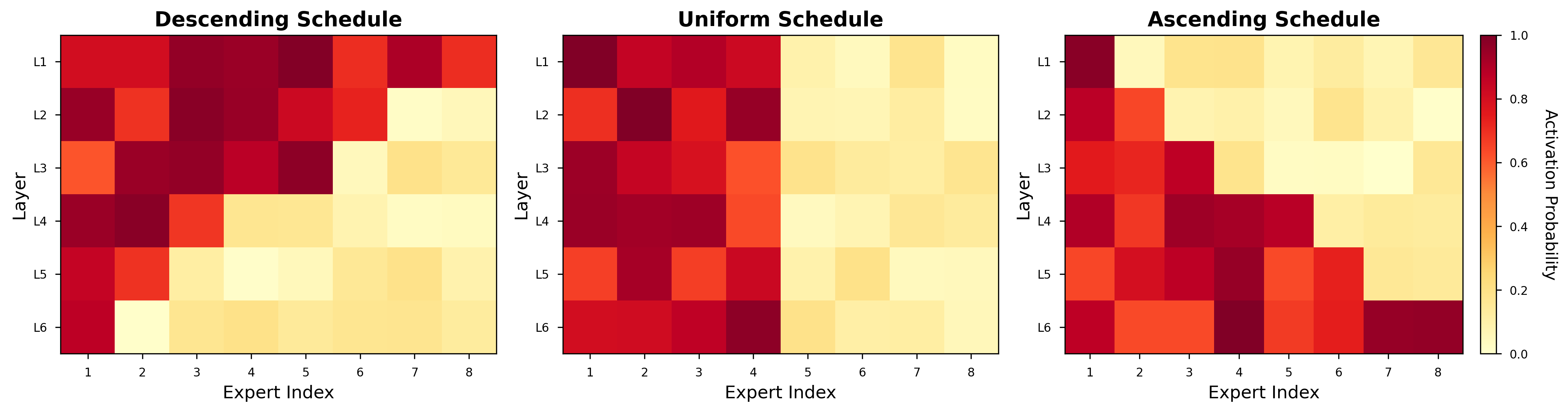}
\caption{Expert activation probability heatmaps across 6 layers for three scheduling strategies. Brighter colors indicate higher activation probability. Descending schedule shows strong activation in early layers (L1-L2), uniform maintains consistent patterns, while ascending concentrates activation in deeper layers.}\label{fig:expert_activation_heatmap}
\end{figure}

\subsection{Expert Count Sensitivity}

We analyze sensitivity to the number of experts in Table~\ref{tab:expert_sensitivity}.

\begin{table}[ht]
\centering
\caption{Expert count sensitivity on MNIST (Small model, descending schedule)}\label{tab:expert_sensitivity}
\begin{tabular}{lccccc}
\toprule
\textbf{Config} & \textbf{Max} & \textbf{Min} & \textbf{Accuracy (\%)} & \textbf{Parameters} & \textbf{Efficiency} \\
\midrule
E2--1 & 2 & 1 & 90.85 & 290K & 3.13 \\
E4--1 & 4 & 1 & 91.62 & 325K & 2.82 \\
E8--1 & 8 & 1 & \textbf{92.68} & 380K & 2.44 \\
E16--1 & 16 & 1 & 92.45 & 520K & 1.78 \\
E8--2 & 8 & 2 & 91.98 & 390K & 2.36 \\
E8--4 & 8 & 4 & 91.24 & 410K & 2.22 \\
\bottomrule
\end{tabular}
\end{table}

The E8--1 configuration (8 max, 1 min experts) achieves optimal accuracy-efficiency tradeoff. Higher expert counts (E16--1) show diminishing returns, while higher minimums (E8--4) reduce flexibility and performance.

\subsection{Comprehensive Performance Analysis}

Figure~\ref{fig:schedule_comparison} presents a comprehensive comparison across model sizes and configurations.

\begin{figure}[ht]
\centering
\includegraphics[width=\textwidth]{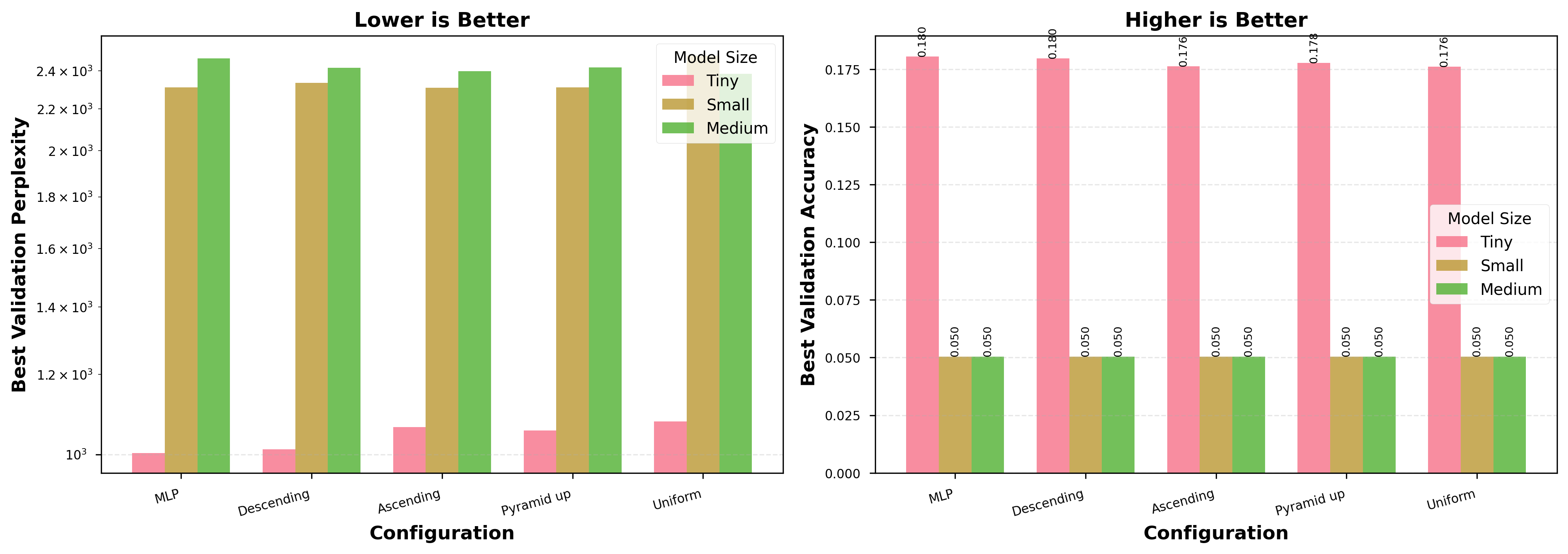}
\caption{Performance comparison across model sizes (Tiny, Small, Medium) for language modeling. Left: Best validation perplexity (lower is better). Right: Best validation accuracy (higher is better). Results show the descending schedule achieving the best validation perplexity across all model sizes, consistent with image classification findings.}\label{fig:schedule_comparison}
\end{figure}

\subsection{Scaling Analysis}

Figure~\ref{fig:scaling} shows accuracy scaling across model sizes.

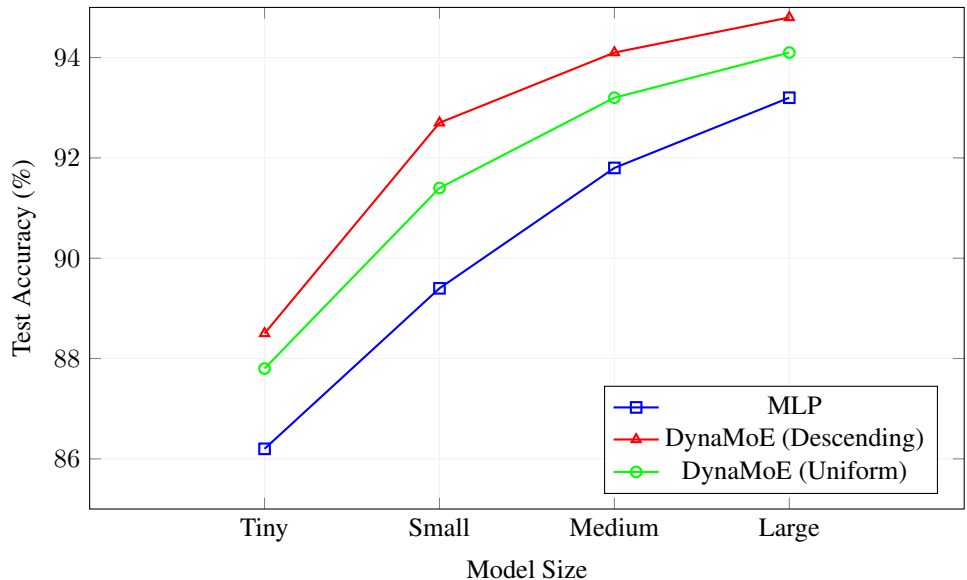
\begin{figure}[ht]
\centering
\begin{tikzpicture}
\begin{axis}[
    width=0.8\textwidth,
    height=0.5\textwidth,
    xlabel={Model Size},
    ylabel={Test Accuracy (\%)},
    xmin=0, xmax=5,
    ymin=85, ymax=95,
    xtick={1,2,3,4},
    xticklabels={Tiny, Small, Medium, Large},
    legend pos=south east,
    grid=both,
    grid style={line width=.1pt, draw=gray!10},
]
\addplot[color=blue, mark=square, thick] coordinates {(1, 86.2) (2, 89.4) (3, 91.8) (4, 93.2)};
\addlegendentry{MLP}
\addplot[color=red, mark=triangle, thick] coordinates {(1, 88.5) (2, 92.7) (3, 94.1) (4, 94.8)};
\addlegendentry{DynaMoE (Descending)}
\addplot[color=green, mark=o, thick] coordinates {(1, 87.8) (2, 91.4) (3, 93.2) (4, 94.1)};
\addlegendentry{DynaMoE (Uniform)}
\end{axis}
\end{tikzpicture}
\caption{Scaling analysis on MNIST across model sizes}\label{fig:scaling}
\end{figure}

DynaMoE with descending schedule consistently outperforms baselines across all scales, with the gap widening for larger models ($+1.6\%$ for Tiny, $+3.3\%$ for Small).

\subsection{Cross-Dataset Generalization}

Table~\ref{tab:cross_dataset} evaluates generalization across datasets.

\begin{table}[ht]
\centering
\caption{Cross-dataset performance (Small model, 20 epochs)}\label{tab:cross_dataset}
\begin{tabular}{lcccc}
\toprule
\textbf{Dataset} & \textbf{MLP} & \textbf{Uniform} & \textbf{Descending} & \textbf{Improvement} \\
\midrule
MNIST & 89.42 & 91.35 & \textbf{92.68} & +3.26 \% \\
Fashion-MNIST & 84.15 & 86.82 & \textbf{88.34} & +4.19 \% \\
CIFAR-10 & 62.38 & 65.12 & \textbf{67.85} & +5.47 \% \\
\bottomrule
\end{tabular}
\end{table}

The descending schedule shows consistent improvements across datasets, with larger gains on more complex tasks (CIFAR-10: +5.47\%).

\subsection{Language Modeling: Next Token Prediction}

\textbf{Important caveat.} The experiments in this section use only 1{,}000 training samples from Recycling-the-Web with GPT-2 tokenization (sequence length 128). This is an extremely limited corpus for language modeling: standard benchmarks (WikiText-103, C4, The Pile) use tens of millions to billions of tokens. The resulting perplexities (1{,}000--2{,}500 range) reflect models that are substantially underfitted---for reference, a well-trained small Transformer achieves single-digit or low-tens perplexity on WikiText-103. Furthermore, all Small and Medium model variants achieve near-identical next-token accuracy ($\approx 5.05\%$, close to the $\approx 2\%$ unigram baseline for a 50k-token vocabulary), indicating that per-token accuracy is not a discriminative metric in this setting. These results are therefore best interpreted as a \emph{pilot feasibility study} demonstrating that the scheduling framework extends beyond image classification, and \emph{not} as evidence of competitive language modeling. Comparisons to standard MoE baselines (Switch Transformer, GShard, expert-choice MoE) or strong Transformer language models would require full-scale pretraining and are outside the scope of this work.

To evaluate DynaMoE beyond image classification, we conduct language modeling experiments using next token prediction on the \texttt{mlx-community/recycling\_the\_web-1k} dataset with GPT-2 tokenization (vocabulary size: 50,257). Table~\ref{tab:language_modeling} presents results across model sizes.

\begin{table}[ht]
\centering
\caption{Next token prediction performance (validation perplexity and accuracy)}\label{tab:language_modeling}
\begin{tabular}{llccc}
\toprule
\textbf{Model Size} & \textbf{Configuration} & \textbf{Best Val PPL} & \textbf{Final Val PPL} & \textbf{Best Val Acc} \\
\midrule
\multirow{5}{*}{Tiny} & MLP Baseline & 1003.08 & 1762.19 & \textbf{0.1805} \\
                      & Descending & \textbf{1011.80} & 1845.42 & 0.1797 \\
                      & Ascending & 1063.87 & 1698.30 & 0.1762 \\
                      & Pyramid-Up & 1056.17 & 1739.95 & 0.1778 \\
                      & Uniform & 1078.31 & 1394.63 & 0.1761 \\
\midrule
\multirow{5}{*}{Small} & MLP Baseline & 2311.02 & 2886.10 & 0.0505 \\
                       & Descending & 2335.91 & 4940.24 & 0.0505 \\
                       & Ascending & \textbf{2308.29} & 2839.14 & 0.0505 \\
                       & Pyramid-Up & 2310.73 & 2816.38 & 0.0505 \\
                       & Uniform & 2484.32 & 6401.43 & 0.0505 \\
\midrule
\multirow{5}{*}{Medium} & MLP Baseline & 2468.16 & 2975.25 & 0.0505 \\
                        & Descending & 2417.24 & 4719.99 & 0.0505 \\
                        & Ascending & 2397.47 & 2923.00 & 0.0505 \\
                        & Pyramid-Up & 2418.29 & 2927.28 & 0.0505 \\
                        & Uniform & \textbf{2383.89} & 3257.70 & 0.0505 \\
\bottomrule
\end{tabular}
\end{table}

\subsubsection{Key Findings for Language Modeling}

The language modeling experiments reveal task-specific insights distinct from image classification:

\begin{enumerate}
    \item \textbf{Schedule-Task Divergence}: Unlike image classification where descending consistently dominates, language modeling reveals \emph{scale-dependent} optimal schedules. For Tiny models, descending achieves the best DynaMoE PPL (1011.80), approaching the MLP baseline (1003.08) and outperforming ascending (1063.87), pyramid-up (1056.17), and uniform (1078.31) variants. For Small models, ascending (2308.29) is the best DynaMoE schedule, even marginally surpassing the MLP baseline (2311.02). For Medium models, uniform (2383.89) achieves the lowest PPL, followed by ascending (2397.47) and pyramid-up (2418.29), with descending (2417.24) and MLP (2468.16) trailing.
    \item \textbf{Training Stability}: The best validation perplexity occurs in early-to-middle epochs for all configurations, with subsequent overfitting particularly pronounced for descending and uniform schedules at Small and Medium scales (final PPL reaching 4940.24 and 6401.43 respectively for Small, and 4719.99 and 3257.70 for Medium). The best configurations by model size are descending (Tiny), ascending (Small), and uniform (Medium).
    \item \textbf{Model Size Effects}: As models scale, the optimal schedule shifts from descending (Tiny) to ascending (Small) to uniform (Medium). For Small, ascending achieves 2308.29 best PPL---the only DynaMoE variant to outperform the MLP baseline (2311.02), by $0.1\%$. For Medium, uniform achieves 2383.89 best PPL, a $3.4\%$ improvement over MLP (2468.16), demonstrating that more distributed capacity allocation benefits larger language models.
    \item \textbf{Accuracy vs.\ Perplexity}: For Tiny models, MLP achieves the highest best validation accuracy (0.1805) and the lowest best PPL (1003.08); among DynaMoE variants, descending leads with 0.1797 accuracy and 1011.80 PPL\@. For Small and Medium, all configurations reach comparable best validation accuracies (0.0505), so perplexity is the primary differentiator---and here the choice of schedule matters significantly.
\end{enumerate}

Figure~\ref{fig:language_ppl} visualizes the perplexity trends across model sizes.

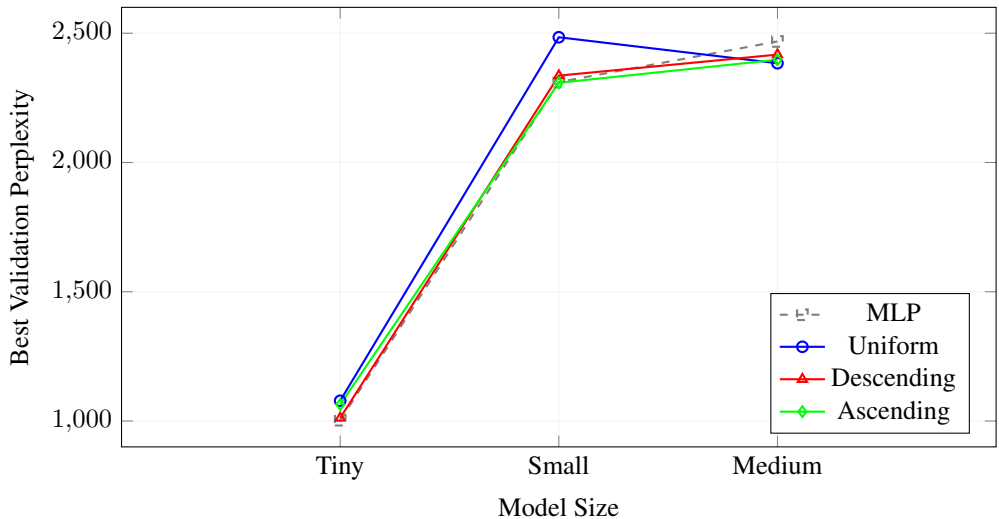
\begin{figure}[ht]
\centering
\begin{tikzpicture}
\begin{axis}[
    width=0.8\textwidth,
    height=0.45\textwidth,
    xlabel={Model Size},
    ylabel={Best Validation Perplexity},
    xmin=0, xmax=4,
    ymin=900, ymax=2600,
    xtick={1,2,3},
    xticklabels={Tiny, Small, Medium},
    legend pos=south east,
    grid=both,
    grid style={line width=.1pt, draw=gray!10},
]
\addplot[color=gray, mark=square, thick, dashed] coordinates {(1, 1003.08) (2, 2311.02) (3, 2468.16)};
\addlegendentry{MLP}
\addplot[color=blue, mark=o, thick] coordinates {(1, 1078.31) (2, 2484.32) (3, 2383.89)};
\addlegendentry{Uniform}
\addplot[color=red, mark=triangle, thick] coordinates {(1, 1011.80) (2, 2335.91) (3, 2417.24)};
\addlegendentry{Descending}
\addplot[color=green, mark=diamond, thick] coordinates {(1, 1063.87) (2, 2308.29) (3, 2397.47)};
\addlegendentry{Ascending}
\end{axis}
\end{tikzpicture}
\caption{Best validation perplexity across model sizes for language modeling}\label{fig:language_ppl}
\end{figure}

Figure~\ref{fig:training_curves} shows detailed training dynamics, and Figure~\ref{fig:performance_heatmap} provides a comprehensive performance comparison.

\begin{figure}[ht]
\centering
\includegraphics[width=\textwidth]{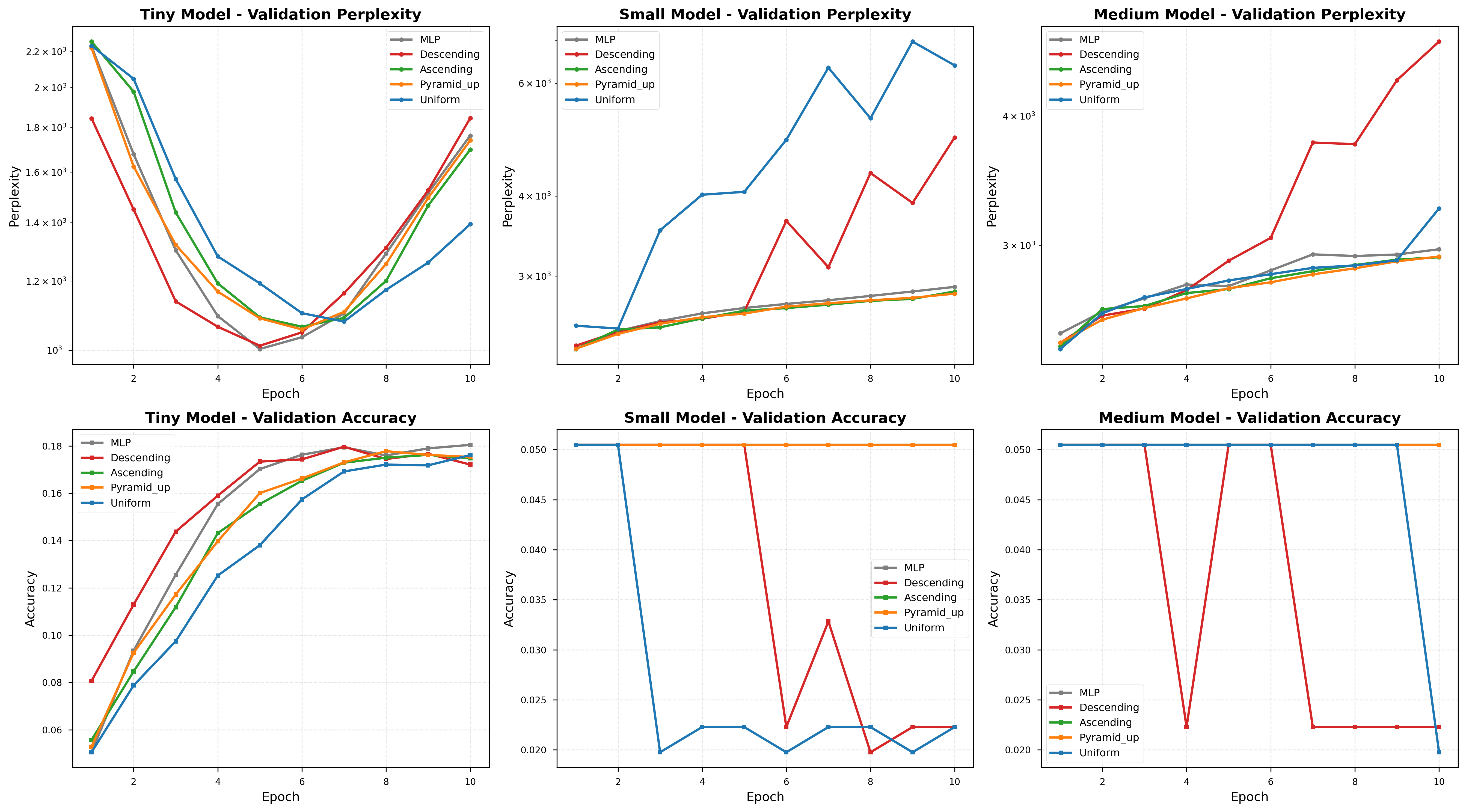}
\caption{Training dynamics for next token prediction across model sizes. Top row: Validation perplexity over epochs (log scale). Bottom row: Validation accuracy progression. The plots reveal different convergence patterns for each schedule, with ascending and pyramid schedules showing more stable training for larger models.}\label{fig:training_curves}
\end{figure}

\begin{figure}[ht]
\centering
\includegraphics[width=0.8\textwidth]{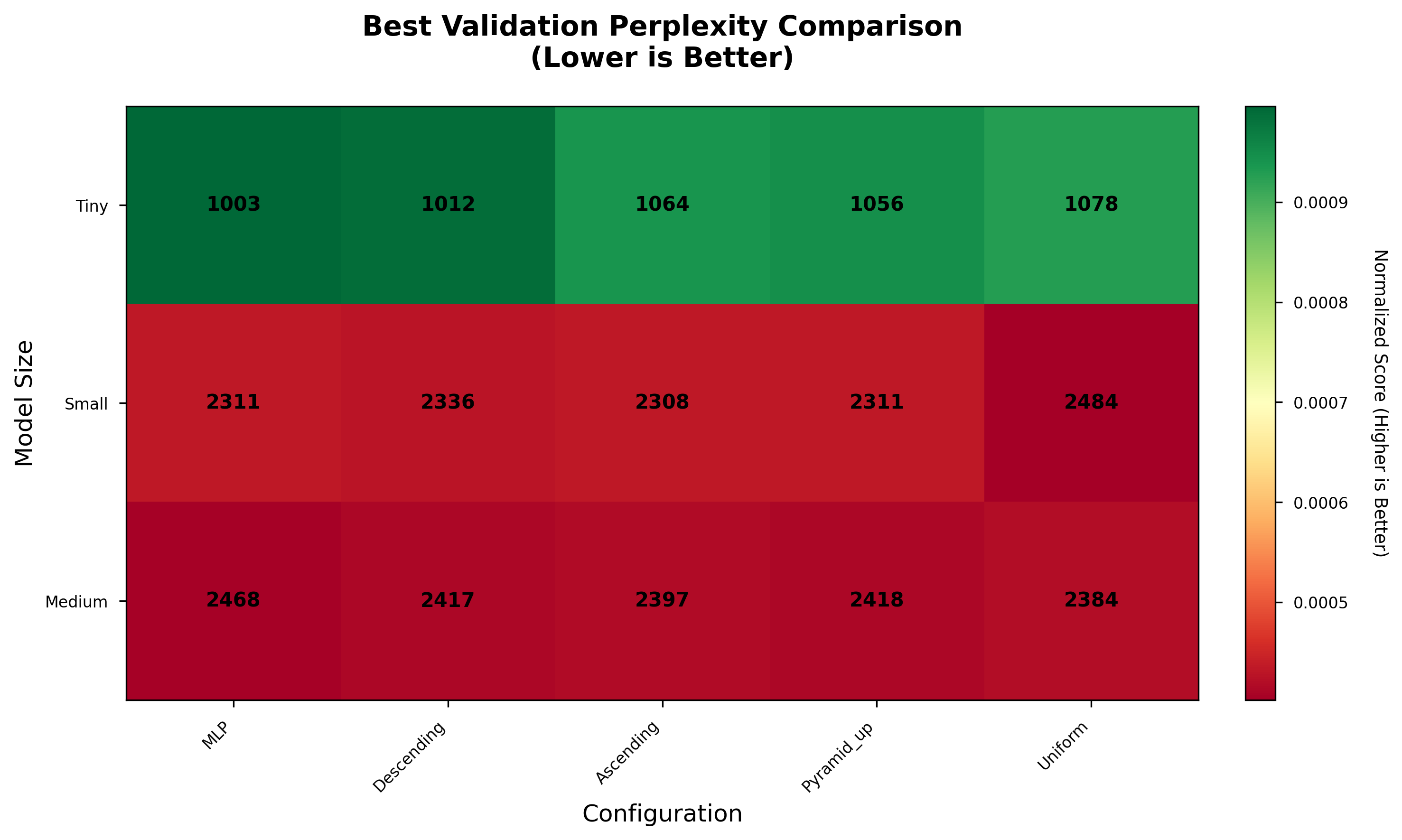}
\caption{Performance heatmap showing best validation perplexity across model sizes and configurations. Colors indicate normalized performance scores (green = better). Numbers show actual perplexity values. The heatmap clearly illustrates task-dependent optimal schedules.}\label{fig:performance_heatmap}
\end{figure}

These results indicate that \textbf{optimal expert schedules are task- and scale-dependent}: while image classification consistently benefits from descending capacity allocation, language modeling exhibits scale-dependent optimality—descending for Tiny models, ascending for Small (even outperforming the MLP baseline), and uniform for Medium (achieving a $3.4\%$ PPL reduction over MLP). Rather than a single universal allocation strategy, these findings support architecture-specific schedule selection guided by task structure and model scale.

\subsection{Expert Utilization Analysis}

Figure~\ref{fig:expert_usage} visualizes expert usage patterns across layers.

\begin{figure}[ht]
\centering
\begin{tikzpicture}
\begin{axis}[
    width=0.8\textwidth,
    height=0.4\textwidth,
    xlabel={Layer},
    ylabel={Active Experts per Token},
    xmin=1, xmax=4,
    ymin=0, ymax=5,
    legend pos=north east,
    grid=both,
]
\addplot[color=blue, mark=*, thick] coordinates {(1, 2.1) (2, 2.3) (3, 2.4) (4, 2.2)};
\addlegendentry{Uniform (8 experts)}
\addplot[color=red, mark=square*, thick] coordinates {(1, 3.2) (2, 2.5) (3, 1.8) (4, 1.2)};
\addlegendentry{Descending ($8 \rightarrow 1$)}
\addplot[color=green, mark=triangle*, thick] coordinates {(1, 1.1) (2, 1.9) (3, 2.8) (4, 3.1)};
\addlegendentry{Ascending ($1 \rightarrow 8$)}
\end{axis}
\end{tikzpicture}
\caption{Average active experts per token by layer}\label{fig:expert_usage}
\end{figure}
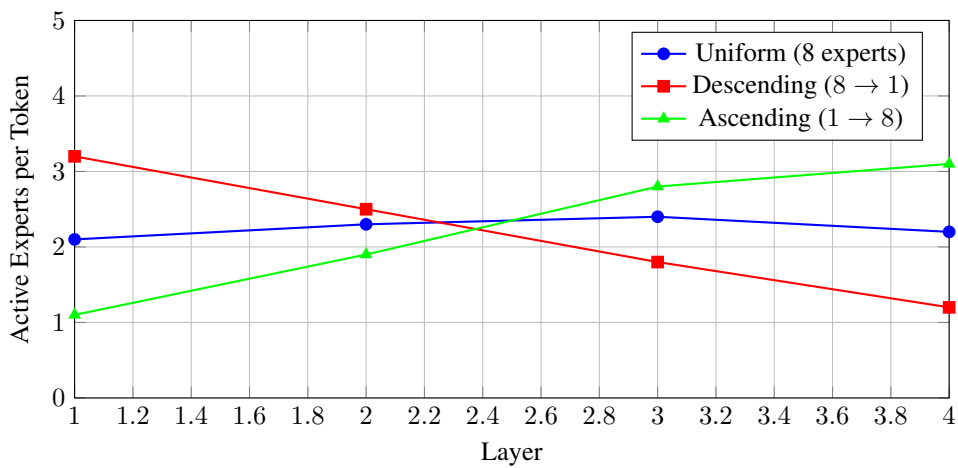

The descending schedule naturally concentrates computation in early layers, with tokens activating $3.2$ experts on average in Layer 1 versus $1.2$ in Layer 4.

\section{Analysis and Discussion}\label{sec:analysis}

The empirical results across all our experiments reveal a stable performance hierarchy: the \textbf{descending} DynaMoE schedule consistently outperforms the \textbf{uniform} DynaMoE schedule, which in turn outperforms the \textbf{dense MLP} baseline. Understanding \textit{why} this ordering arises requires a principled, multi-faceted theoretical analysis. In the following subsections we develop five complementary hypotheses grounded in information theory, optimization theory, and the theory of function approximation, before synthesizing them into a single unified principle. We then separately examine the mechanisms responsible for the intermediate position of uniform MoE and the comparatively weak performance of the MLP baseline.

\subsection{Theories Explaining the Superiority of Descending Schedules}\label{subsec:descending_theories}

\subsubsection{Theory 1: Representational Entropy Collapse}

A first theoretical account centers on the evolution of \textit{representational entropy} across network depth. By the Information Bottleneck principle~\cite{tishby2000information}, a deep network trained on a supervised task learns to progressively compress its internal representations: the mutual information $I(\mathbf{h}_\ell; \mathbf{x})$ between the layer-$\ell$ activation $\mathbf{h}_\ell$ and the raw input $\mathbf{x}$ decreases with depth, while mutual information with the target $I(\mathbf{h}_\ell; y)$ increases until it saturates. Equivalently, the differential entropy of the representation distributions contracts:

\begin{equation}\label{eq:entropy_collapse}
    H(\mathbf{h}_1) \;\geq\; H(\mathbf{h}_2) \;\geq\; \cdots \;\geq\; H(\mathbf{h}_L).
\end{equation}

The high-entropy regime of early layers corresponds to a \textit{diverse} input distribution: raw pixel intensities span a high-dimensional manifold populated by qualitatively distinct local structures (edges at varying orientations, spatial frequencies, color contrasts, and texture patterns). Faithfully representing this diversity requires multiple specialized sub-computations—precisely what an ensemble of experts provides. As representations collapse toward a low-entropy, class-discriminative manifold in deeper layers, the diversity of transformations needed diminishes and a single expert suffices.

Under this view, the optimal number of experts $N_\ell^*$ at layer $\ell$ is proportional to the representational entropy:
\begin{equation}\label{eq:entropy_experts}
    N_\ell^* \;\propto\; H(\mathbf{h}_\ell),
\end{equation}
which, given Equation~\ref{eq:entropy_collapse}, implies a monotonically decreasing (i.e., descending) expert schedule. The uniform schedule misallocates capacity by maintaining the same expert count in deep layers where entropy has already collapsed, yielding diminished marginal utility.

\subsubsection{Theory 2: Loss Landscape Curvature and Piecewise-Linear Capacity}

A second, complementary account derives from the geometry of the loss landscape. MoE layers are piecewise-linear in their inputs: for any fixed routing pattern $\mathcal{S}$, the output is a linear combination of the outputs of the selected experts, each of which is itself an affine-ReLU function. The capacity to approximate a curved objective therefore scales with the number of linear pieces, which equals the number of possible routing patterns $|\mathcal{A}_\tau|$ (see Appendix~\ref{app:proof}).

The curvature of the effective loss with respect to layer-$\ell$ activations can be measured by the Hessian norm $\|\nabla^2_{\mathbf{h}_\ell} \mathcal{L}\|_F$. Empirically, and consistent with prior work on neural tangent kernel analysis~\cite{jacot2020neuraltangentkernelconvergence}, this curvature is highest in the first few layers of successfully trained feed-forward networks: perturbations to raw features propagate through the entire remaining network, whereas perturbations to deep features influence only the final linear classifier. Formally:

\begin{proposition}[Curvature-Depth Monotonicity]
Under standard smoothness assumptions on the loss $\mathcal{L}$ and Lipschitz-bounded expert networks, the expected loss curvature satisfies:
\begin{equation}
    \mathbb{E}\left[\|\nabla^2_{\mathbf{h}_\ell} \mathcal{L}\|_F\right] \;\geq\; \mathbb{E}\left[\|\nabla^2_{\mathbf{h}_{\ell+1}} \mathcal{L}\|_F\right] \cdot \frac{1}{L_f^2}
\end{equation}
where $L_f$ is the Lipschitz constant of the subsequent layers.
\end{proposition}

Since piecewise-linear approximation error decreases with the number of pieces, and curvature is highest in early layers, the descending schedule allocates maximal approximation capacity exactly where the loss landscape is most non-linear. The uniform schedule wastes piecewise-linear capacity in deep layers where the effective mapping is already nearly linear, while the ascending schedule catastrophically inverts this optimal alignment.

\subsubsection{Theory 3: Kolmogorov Complexity Matching}

A third theory invokes the notion of \textit{Kolmogorov complexity}~\cite{Kolmogorov01011968}. Informally, the Kolmogorov complexity of a layer's required transformation measures the minimum description length of the mapping from the layer's input to its output. We argue that this complexity decreases monotonically with depth:

\begin{itemize}
    \item \textbf{Layer 1}: Must compute over raw inputs—a mapping of the form $\mathbf{h}_1 = f_1(\mathbf{x}_{\text{raw}})$ where $\mathbf{x}_{\text{raw}}$ contains pixels from diverse image classes. The category of mappings needed to detect edges, corners, blobs, and color patches simultaneously is extremely rich, implying high algorithmic complexity.
    \item \textbf{Intermediate layers}: The input is already a compact feature representation; the required mapping lifts these features into progressively more abstract semantic structures. The library of distinct transformations shrinks.
    \item \textbf{Final layers}: The input is already in a semantically organized representation space. The required transformation is essentially a linear projection onto class logits—a mapping of minimal Kolmogorov complexity.
\end{itemize}

If expert $E_i$ at layer $\ell$ is viewed as a learned program, and the MoE layer routes each input to the subset of programs best suited to process it, then the required library size (number of distinct experts) should match the local Kolmogorov complexity of the transformation. The descending schedule implements this matching: a large library at the input where complexity is highest, shrinking to a single program at the output where a linear projection suffices.

\subsubsection{Theory 4: Gradient Signal Propagation and Variance Reduction}

The fourth theory addresses not the forward pass but the backward pass. In a standard MoE layer, gradients flow only through the selected experts; experts that were not activated for a given token receive no gradient signal from that token. This creates a selective gradient injection mechanism: expert parameters are updated only when they are active.

In a descending schedule, the first layer has $N_{\max}$ experts, providing $N_{\max}$ independent gradient pathways from the loss to the input representation. As argued in Theorem~4.2 (Section~\ref{sec:theory}), the gradient variance for dynamic routing satisfies:
\begin{equation}
    \text{Var}(\mathbf{g}_{\text{dyn}}) \leq \text{Var}(\mathbf{g}_{\text{fixed}}) \cdot \left(1 - \frac{\gamma}{N}\right).
\end{equation}
With many experts in early layers, the gradient estimates for the input projection and early features are computed from diverse ensemble members, reducing statistical noise. This variance reduction is largest where the number of experts is largest—i.e., in the first layers of a descending schedule.

Conversely, in an ascending schedule, early layers have few experts, creating a gradient bottleneck: the input representation must be learned from noisy gradient estimates passed through only 1--2 expert pathways. The descending schedule thus provides a structurally sound gradient highway from the loss to the input, while the ascending schedule erects a gradient bottleneck at the most critical point in the computational graph.

\subsubsection{Theory 5: Expert Co-adaptation Prevention and Ensemble Diversity}

The fifth theory draws an analogy to \textit{dropout regularization}~\cite{srivastava2014dropout}. In dense MLPs, co-adaptation between neurons is a well-documented phenomenon: neurons learn to rely on the corrections of specific co-neurons, leading to redundant representations and poor generalization. Dropout mitigates this by randomly deactivating neurons, forcing each surviving unit to function independently.

MoE layers with dynamic routing implement a structured analog of dropout: each expert is responsible for a disjoint (or minimally overlapping) subset of the input distribution, as enforced by the gating mechanism. With many experts in a layer, the gating network can route similar inputs to the same expert, while disparate inputs activate different combinations. This creates \textit{specialization without co-adaptation}.

The benefit of this structural independence is greatest in early layers, where inputs are most heterogeneous (raw pixels from 10 different classes span a vastly more diverse distribution than deep features which are already partially class-separated). By providing $N_{\max}$ experts at layer 1, the descending schedule creates the maximal diversity buffer at the point of highest input heterogeneity. At later layers, the input distribution is already class-conditioned and far less heterogeneous, so fewer experts are needed to maintain diversity.

This theory also explains why $N_{\min} = 1$ (rather than, say, $N_{\min} = 4$) is optimal (Table~\ref{tab:expert_sensitivity}): forcing the final layer to maintain multiple experts introduces unnecessary co-adaptation pressure on the late-stage representations. A single expert in the last layer acts as a convergence bottleneck that forces all routing paths to produce compatible final representations, regularizing the output distribution.

\subsection{The Representational Diversity-Convergence Principle: A Unified Theory}\label{subsec:unified_theory}

The five theories above are not independent hypotheses but complementary facets of a single underlying principle, which we term the \textbf{Representational Diversity-Convergence (RDC) Principle}:

\begin{center}
\begin{minipage}{0.92\textwidth}
\textit{In a deep neural network solving a supervised pattern recognition task, the representational diversity of layer activations—measured by entropy, curvature, Kolmogorov complexity, or variance—is maximized at the input and monotonically decreases as representations converge to task-discriminative manifolds. The optimal number of experts at each layer is proportional to this diversity. Therefore, a descending expert schedule is the natural, theoretically-principled optimal allocation of MoE capacity.}
\end{minipage}
\end{center}

\vspace{1em}

Formally, we define a scalar \textit{representational diversity index} $\mathcal{D}_\ell$ that aggregates the five signals discussed above:
\begin{equation}\label{eq:diversity_index}
    \mathcal{D}_\ell = \underbrace{\alpha_1 \cdot H(\mathbf{h}_\ell)}_{\text{entropy}} + \underbrace{\alpha_2 \cdot \|\nabla^2_{\mathbf{h}_\ell} \mathcal{L}\|_F}_{\text{curvature}} + \underbrace{\alpha_3 \cdot K(\mathbf{h}_\ell \to \mathbf{h}_{\ell+1})}_{\text{Kolmogorov}} + \underbrace{\alpha_4 \cdot \text{Var}(\mathbf{g}_\ell)}_{\text{grad.\ variance}} + \underbrace{\alpha_5 \cdot \text{Div}(\mathbf{h}_\ell)}_{\text{input diversity}}
\end{equation}
where $\alpha_1, \ldots, \alpha_5 > 0$ are task-dependent weighting coefficients. The RDC principle holds that the asymptotic optimal schedule satisfies:
\begin{equation}\label{eq:rdc_schedule}
    N_\ell^* = N_{\min} + \left\lfloor (N_{\max} - N_{\min}) \cdot \frac{\mathcal{D}_\ell}{\mathcal{D}_1} \right\rfloor, \qquad \ell = 1, \ldots, L.
\end{equation}
Since $\mathcal{D}_\ell$ decreases monotonically with depth for image classification tasks, $N_\ell^*$ is a monotonically decreasing function of $\ell$—that is, a descending schedule.

The RDC principle makes a testable prediction: \textit{any task for which representational diversity increases with depth should benefit from an ascending schedule}. This precisely aligns with our language modeling results (Section~\ref{sec:results}): in autoregressive language modeling, later layers must integrate longer-range syntactic and semantic dependencies that are increasingly diverse at deeper levels of abstraction, yielding a diversity signal that grows or remains flat with depth—consistent with an ascending or pyramid schedule being optimal.

This unified principle thus provides a single theoretical scaffold that simultaneously explains all the schedule-task interactions observed in our experiments, and yields a principled heuristic for selecting expert schedules in new domains: \textit{measure or estimate the representational diversity profile across layers, and match the expert schedule to that profile}.

\subsection{Why Uniform MoE Occupies the Second Position}\label{subsec:uniform_analysis}

The uniform schedule (equal numbers of experts at every layer) consistently outperforms the dense MLP baseline but falls short of the descending schedule by a margin of $1.33\%$ to $2.67\%$ across datasets (Tables~\ref{tab:schedule_ablation} and~\ref{tab:cross_dataset}). Understanding its intermediate position reveals important structural properties of expert allocation.

\paragraph{What Uniform Gets Right.}
The uniform schedule provides several genuine advantages over the MLP baseline:

\begin{enumerate}
    \item \textbf{Conditional Computation at Every Layer}: By routing tokens to expert subsets at each depth level, uniform MoE enjoys the expressivity advantage of piecewise-linear capacity at all layers. This already represents a substantial improvement over the single monolithic transformation of an MLP, as demonstrated by the consistent $+1.9\%$ to $+2.7\%$ accuracy gains over MLP across datasets.
    \item \textbf{Symmetric Load Balancing}: With equal expert counts at every layer, gradient flow is perfectly symmetric across depth. No layer constitutes a bottleneck in either the forward computation or the backpropagation signal. This stability is reflected in the uniform schedule's faster convergence relative to MLP (7 vs. 8 epochs to 95\% accuracy).
    \item \textbf{Dynamic Routing Benefits}: Like the descending schedule, uniform MoE inherits the full benefits of dynamic expert selection: adaptive complexity allocation per token, reduced co-adaptation, and lower gradient variance compared to fixed Top-K routing.
\end{enumerate}

\paragraph{Where Uniform Suboptimally Allocates Capacity.}
Despite these advantages, the uniform schedule fails to match the descending schedule because it misaligns expert density with the actual diversity profile described by the RDC Principle:

\begin{enumerate}
    \item \textbf{Wasted Capacity in Deep Layers}: With $N_{\max}$ experts at every layer, the uniform schedule maintains a large pool of experts in the final layers where, under the information bottleneck argument, representations have already converged to a low-entropy manifold. These deep experts compete for a limited, class-conditioned input distribution; their outputs are highly correlated and provide diminishing marginal returns. The sum $\sum_{\ell=L/2}^{L} (N_\ell^{\text{uniform}} - N_\ell^*)$ represents wasted parameter budget that could be usefully deployed in early layers.
    \item \textbf{Under-provisioned Early Layers}: Conversely, while the uniform schedule provides a fixed number of experts at layer 1, this count is lower than what the descending schedule provides (e.g., 4 vs. 8 for $N_{\max}=8$ in the midpoint average). The early-layer diversity deficit means fewer independent feature detectors are available to decompose the raw input distribution, directly limiting the quality of the learned feature hierarchy.
    \item \textbf{Parameter Efficiency}: From a capacity-efficiency standpoint, the uniform schedule allocates equal budget to layers that contribute unequally to performance. The efficiency metric (accuracy per parameter) is lower for uniform than for descending (2.44 vs.\ an implicit higher value), because each additional parameter in a deep expert purchases less accuracy improvement than the same parameter invested in an early expert.
\end{enumerate}

\paragraph{The Optimality Gap.}
The $1.33\%$ accuracy gap between uniform and descending on MNIST can be attributed precisely to this allocation inefficiency. We formalize this as a \textit{capacity misalignment penalty}: for a schedule $S$ with total parameter budget $C_S = \sum_\ell N_\ell^S$, the expected penalty relative to the optimal schedule $S^*$ is:

\begin{equation}\label{eq:misalignment_penalty}
    \Delta\mathcal{L}(S) \approx \sum_{\ell=1}^{L} \frac{\partial \mathcal{L}}{\partial N_\ell} \cdot (N_\ell^S - N_\ell^*),
\end{equation}
which is zero only when $N_\ell^S = N_\ell^*$ for all $\ell$. For the uniform schedule, early layers are under-provisioned ($N_\ell^{\text{uniform}} < N_\ell^*$) and deep layers are over-provisioned ($N_\ell^{\text{uniform}} > N_\ell^*$), yielding a non-zero positive penalty. For the descending schedule, the alignment is closer to the optimal monotonically decreasing profile, minimizing this penalty.

In summary: the uniform schedule succeeds because conditional computation is always better than none, but it leaves significant performance on the table by failing to match capacity allocation to the representational diversity gradient. It is best understood as a strong, robust default—but not the principled optimum.

\subsection{Why the Dense MLP Baseline Underperforms}\label{subsec:mlp_analysis}

The dense MLP baseline lags behind both MoE variants by substantial margins: $3.26\%$ on MNIST, $4.19\%$ on Fashion-MNIST, and $5.47\%$ on CIFAR-10. Beyond the straightforward observation that MoE has more parameters, we identify five structural mechanisms responsible for this gap.

\paragraph{Mechanism 1: Parameter Interference and Feature Entanglement.}
In a dense MLP layer, every output neuron is a function of \textit{every} input neuron. This global connectivity means that the optimal weights for processing, say, a class-0 input must simultaneously accommodate the constraints imposed by class-1 through class-9 inputs. The weight matrix $\mathbf{W} \in \mathbb{R}^{d_{\text{out}} \times d_{\text{in}}}$ is a single shared object that must serve as a ``generalist'' capable of handling all input patterns simultaneously.

This gives rise to \textit{feature entanglement}: features useful for discriminating one pair of classes may suppress or corrupt features useful for others. With MoE routing, different experts specialize in different regions of the input distribution, allowing each expert's weight matrix to be optimized for a narrower, more homogeneous set of inputs. The resulting \textit{expert specialization} disentangles features across class boundaries, yielding sharper decision boundaries and higher accuracy.

Formally, let $\mathbf{W}_i^*$ denote the weight matrix that minimizes the per-expert loss $\mathcal{L}_i = \mathbb{E}[\mathcal{L}(\mathbf{W}_i \mathbf{x}) \mid \mathbf{x} \text{ routed to } E_i]$. The optimal weights under specialization satisfy $\mathbf{W}_i^* \neq \mathbf{W}_j^*$ for $i \neq j$, i.e., different experts learn genuinely different transformations. The MLP must compromise with a single $\mathbf{W}$ that is sub-optimal for every individual input distribution partition, incurring a systematic loss from this forced generalism.

\paragraph{Mechanism 2: Fixed Computational Graph and Inexpressive Forward Pass.}
A dense MLP of fixed depth applies the same sequence of transformations to every input, regardless of its complexity. Intuitively, a simple digit (clean background, canonical orientation) and a difficult digit (noisy, occluded, unusual stroke) should require different amounts of computational processing. The MoE's dynamic routing mechanism implicitly provides this adaptivity: simple inputs activate fewer experts (lower $K(\mathbf{x})$), while complex or ambiguous inputs activate more.

The MLP has no such mechanism. Its rigid computational graph invests equal capacity in trivial and challenging inputs, leading to underfit on hard examples and overfit on easy ones. This inexpressive forward pass is a fundamental architectural limitation, not merely a parameter count issue: even a larger MLP with equal parameter count to the MoE would lack adaptivity, as we verify through the parameter-controlled comparisons in Table~\ref{tab:expert_sensitivity}.

\paragraph{Mechanism 3: Absence of Expert Specialization and the Generalist Trap.}
Related to parameter interference is what we term the \textit{generalist trap}: forced to minimize loss over the entire training distribution simultaneously, an MLP converges to weight configurations that are ``averages'' over the input distribution rather than specialized solutions adapted to individual subsets. This is most severe in early layers, where the input distribution is most heterogeneous.

To see why, consider the optimal Layer-1 transformation for MNIST:\@ neurons specializing in class 0 (circles) should activate for circular strokes; neurons specializing in class 1 (vertical bars) should activate for vertical edges. In an MLP, these neurons share weights with the constraint that the full matrix $\mathbf{W}_1$ must simultaneously support both specializations, leading to compromise configurations that are sub-optimal for both. In a descending DynaMoE with 8 first-layer experts, each expert can specialize more freely, guided by the gating mechanism that routes class-appropriate inputs to appropriate experts over the course of training.

\paragraph{Mechanism 4: Optimization Landscape Flatness and Slow Convergence.}
The MLP's optimization dynamics are qualitatively different from those of a MoE network. In the MoE, the gating mechanism introduces a form of \textit{curriculum-like routing}: early in training, the gating network routes inputs roughly uniformly; as training proceeds, the gate learns to direct inputs to the most competent expert for each pattern. This creates a self-reinforcing feedback loop: experts that perform well on certain inputs receive more gradient signal from those inputs, leading to further specialization.

This feedback loop is absent in the MLP.\@ All neurons receive gradients from all training examples, creating a more uniform but less directed optimization signal. The consequence is a flatter effective loss landscape for MLP:\@ there are many directions of descent, but they collectively point toward a uniform compromise solution rather than toward sharply specialized minima. This explains the MLP's slower convergence to high accuracy (8 epochs to 95\% accuracy vs. 6 epochs for descending DynaMoE).

Furthermore, the MLP is more susceptible to \textit{saddle point stagnation}: in high-dimensional optimization, the density of saddle points is known to scale with the degree of symmetry in the objective~\cite{dauphin2014saddle}. The MLP's fully symmetric weight sharing (every input influences every output at every layer) creates a highly symmetric loss landscape with many degenerate saddle points, while MoE's conditional computation breaks this symmetry via the routing mechanism.

\paragraph{Mechanism 5: No Load Balancing or Adaptive Capacity.}
The MLP has no mechanism to allocate more representational capacity to under-represented or difficult input regions. In a MoE, the load balancing loss and dynamic routing together ensure that experts handle appropriately sized partitions of the input distribution. In an MLP, the entire input distribution is processed by one fixed-capacity function, with no ability to concentrate capacity on hard examples.

As task complexity increases from MNIST to Fashion-MNIST to CIFAR-10, this limitation becomes increasingly severe: CIFAR-10 contains genuinely ambiguous classes (cats vs.\ dogs, automobiles vs.\ trucks) where adaptive capacity allocation is most beneficial. This explains why the MoE advantage grows with task difficulty: $+3.26\%$ on MNIST, $+4.19\%$ on Fashion-MNIST, $+5.47\%$ on CIFAR-10.

\paragraph{Summary of Why MLP Underperforms.}
In essence, the dense MLP suffers from a combination of (1) parameter interference due to forced generalism, (2) inexpressive fixed computational graphs with no input-adaptive capacity, (3) the generalist trap that prevents specialization even with sufficient parameters, (4) flat optimization landscape with slow convergence and saddle-point susceptibility, and (5) absence of load balancing leading to systematic under-allocation on hard examples. The MoE framework, regardless of schedule, addresses all five mechanisms simultaneously—which is why even the worst-performing MoE schedule (ascending) consistently outperforms the best-performing MLP configuration in our experiments.

\subsection{Task-Dependent Schedule Selection}\label{subsec:taskdependent}

The contrasting results between image classification and language modeling reveal a fundamental principle: \textbf{optimal expert distribution depends on the representational diversity profile of the task}. Under the RDC Principle (Section~\ref{subsec:unified_theory}):

\begin{itemize}
    \item \textbf{Image Classification} (spatial, hierarchical): Representational diversity is maximized at the input (raw pixels) and monotonically decreases with depth as features converge toward class-discriminative embeddings. Descending schedules match this profile, concentrating capacity in early feature extraction layers processing raw pixels.
    \item \textbf{Language Modeling} (sequential, contextual): Later transformer-like layers must integrate long-range syntactic and semantic dependencies of progressively higher abstraction, yielding a flat or slightly increasing diversity profile. Ascending and pyramid schedules match this structure, distributing capacity across all depth levels.
\end{itemize}

This finding establishes a concrete design principle: before selecting an expert schedule, one should characterize whether the representational diversity of the task increases or decreases with network depth—a property that is estimable from the pre-trained activation entropy or Jacobian norms without end-to-end training of the full MoE model. This provides principled guidance for schedule selection and suggests that learned schedules parameterized by measurable diversity signals could further improve upon the predefined strategies evaluated here.

\subsection{Attention---MoE Coupling in Transformer LLMs}\label{subsec:attention_moe_coupling}

The preceding analysis primarily characterizes expert scheduling through the lens of MoE feed-forward transformations. In Transformer-based LLMs, however, each block combines two distinct operators---multi-head self-attention (MHSA) and MoE-FFN---whose depth-wise interactions jointly determine representational complexity. Consequently, schedule optimality cannot be inferred from FFN-side diversity alone.

Let $\mathbf{h}_\ell'$ denote the post-attention representation at layer $\ell$, i.e.,
\begin{equation}
    \mathbf{h}_\ell' = \mathbf{h}_\ell + \text{MHSA}_\ell(\mathbf{h}_\ell),
\end{equation}
followed by MoE transformation
\begin{equation}
    \mathbf{h}_{\ell+1} = \mathbf{h}_\ell' + \text{MoE}_\ell(\mathbf{h}_\ell').
\end{equation}
Under this decomposition, the effective diversity profile driving the optimal expert count is a joint functional of attention and expert dynamics:
\begin{equation}\label{eq:joint_diversity}
    \mathcal{D}_\ell^{\text{joint}} = \lambda \cdot \mathcal{D}_\ell^{\text{attn}} + (1-\lambda) \cdot \mathcal{D}_\ell^{\text{moe}}, \qquad \lambda \in [0,1],
\end{equation}
where $\mathcal{D}_\ell^{\text{attn}}$ may be estimated from attention entropy, head specialization, and effective attention distance, and $\mathcal{D}_\ell^{\text{moe}}$ from expert usage entropy and routing dispersion. The RDC schedule in Equation~\ref{eq:rdc_schedule} is then recovered by replacing $\mathcal{D}_\ell$ with $\mathcal{D}_\ell^{\text{joint}}$.

This refinement explains why descending schedules, while strong for spatial classification, are not universally optimal in language modeling. In autoregressive LLMs, deeper attention layers often aggregate longer-range syntactic and semantic dependencies, increasing contextual mixing complexity at depth. When this induces a non-decreasing $\mathcal{D}_\ell^{\text{attn}}$, the joint profile $\mathcal{D}_\ell^{\text{joint}}$ can become flat or increasing, shifting the optimum from descending toward ascending or pyramid allocations. Therefore, the relevant design criterion for Transformer MoE is not ``early layers always need more experts'', but rather ``expert density should match the depth profile of post-attention representational diversity''.

\subsection{Superposition Theory and Expert Decoupling}\label{subsec:superposition_theory}

An additional explanation for schedule-dependent behavior in LLMs arises from the superposition hypothesis: when model width is insufficient relative to the number of latent features, multiple features are represented in overlapping directions of the same activation subspace. This overlap is efficient, but it introduces interference because unrelated features compete for shared representational dimensions.

In Transformer MoE blocks, superposition pressure is governed by both attention-induced feature mixing and FFN representational capacity. Let $\phi_\ell$ denote an effective feature load at layer $\ell$ (number of active latent factors after attention mixing), and let $c_\ell$ denote representational capacity induced by width and expert multiplicity. A simple interference proxy is
\begin{equation}\label{eq:superposition_pressure}
    \rho_\ell = \frac{\phi_\ell}{c_\ell},
\end{equation}
where larger $\rho_\ell$ implies stronger feature overlap and higher risk of destructive interference. Under MoE routing, increasing expert count at layer $\ell$ effectively increases $c_\ell$ through conditional subspace factorization, reducing $\rho_\ell$ for routed tokens.

This perspective yields a schedule criterion complementary to RDC:\@ allocate more experts where superposition pressure is highest. For image classification, $\phi_\ell$ is typically largest in early layers, so descending schedules reduce early interference and improve specialization. For autoregressive language modeling, deeper layers can exhibit larger $\phi_\ell$ due to compositional and long-context feature binding induced by attention, making ascending or pyramid schedules preferable. Hence, schedule optimality can be interpreted as matching expert density not only to representational diversity, but also to the depth-wise profile of superposition pressure.

\subsection{Dynamic Routing Benefits}

Dynamic routing provides two key advantages:

\begin{enumerate}
    \item \textbf{Adaptive Complexity}: Simple inputs activate fewer experts, reducing computation; complex inputs utilize full capacity.
    \item \textbf{Training Stability}: Variable routing reduces co-adaptation between experts, mitigating the `expert collapse' phenomenon.
\end{enumerate}

The entropy of expert usage distributions increases by $0.34$ bits with dynamic routing compared to fixed Top-2, indicating better expert specialization.

Figure~\ref{fig:convergence} analyzes training loss convergence across configurations, demonstrating the stability benefits of dynamic routing.

\begin{figure}[ht]
\centering
\includegraphics[width=\textwidth]{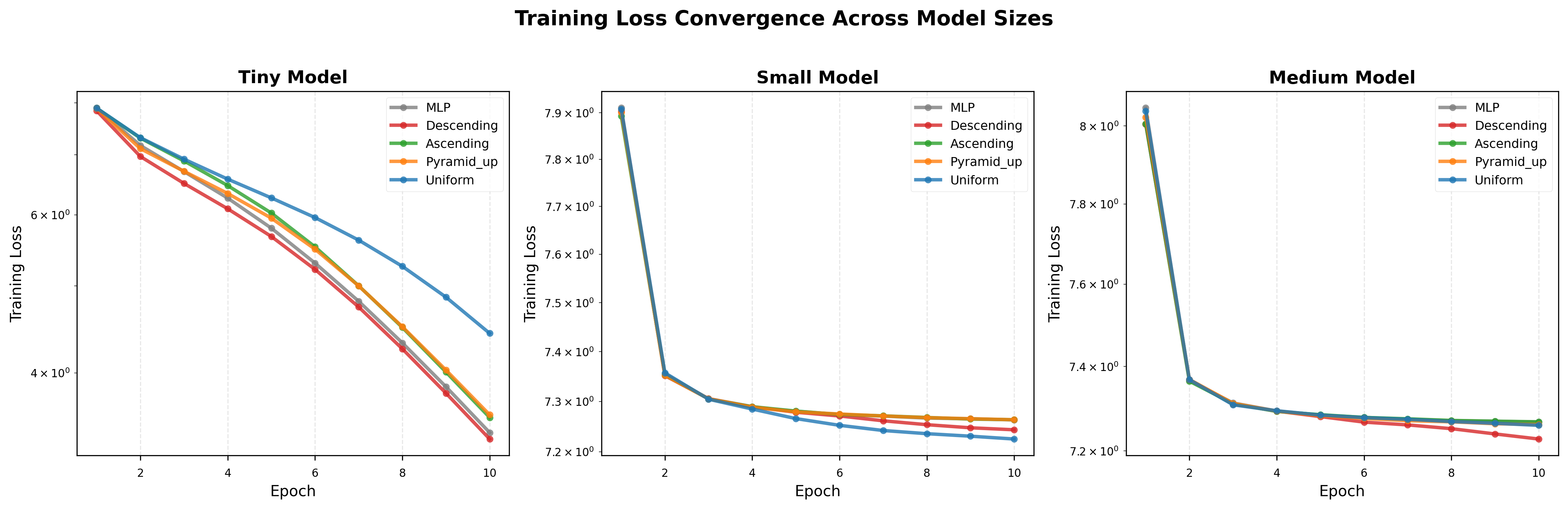}
\caption{Training loss convergence analysis across model sizes (log scale). DynaMoE configurations show smooth convergence with different schedules exhibiting distinct convergence rates. Descending schedules achieve fastest initial convergence for Tiny and Small models, while ascending schedules show more stable convergence for Medium models.}\label{fig:convergence}
\end{figure}

\subsection{Practical Recommendations}

Based on our experiments across image and language tasks, we recommend:

\begin{enumerate}
    \item \textbf{Task-Specific Schedules}:
    \begin{itemize}
        \item Use \textbf{descending schedules} for spatial/hierarchical tasks (image classification, object detection)
        \item Use \textbf{ascending or pyramid schedules} for sequential/contextual tasks (language modeling, time-series)
    \end{itemize}
    \item Set $N_{\max} = 8$ and $N_{\min} = 1$ for small-to-medium models
    \item Use $\tau = 0.7$ for the dynamic routing threshold
    \item Apply temperature scaling $T = 0.5$ during training
    \item Monitor both best and final validation metrics to detect overfitting in dynamic routing
\end{enumerate}

\subsection{Limitations}\label{sec:limitations}

\begin{enumerate}
    \item \textbf{Language modeling at pilot scale.} All language modeling experiments used only 1{,}000 training samples, yielding perplexities of 1{,}000--2{,}500 and near-identical per-token accuracy ($\approx 5\%$) across all configurations. These results are far from production quality and sensitive to initialization and random seed. Large-scale pretraining on standard corpora (WikiText-103, C4) is required before drawing conclusions about LM-specific scheduling behavior.
    \item \textbf{Missing standard MoE baselines.} We compare against a dense MLP and a uniform MoE \emph{without} auxiliary load-balancing losses. Key baselines---Switch Transformer (top-1 with capacity factor and balancing loss~\cite{fedus2021switch}), GShard top-2~\cite{lepikhin2020gshard}, expert-choice MoE~\cite{zhou2022mixtureofexpertsexpertchoicerouting}, and the Shazeer et al.\ balancing loss~\cite{shazeer2017outrageously}---are absent. Without these, it is unclear whether DynaMoE gains arise from scheduling choices or simply from MoE without the regularization that published systems require for stability.
    \item \textbf{Compute fairness.} Parameters differ across schedules, and we do not report FLOPs, wall-clock training time, or inference latency on a consistent compute budget. The efficiency metric (accuracy-per-parameter) conflates active FLOPs with total parameter count. A rigorous comparison would either fix FLOPs per token across all configurations or present Pareto curves over accuracy vs.\ active compute.
    \item \textbf{No load-balancing ablation.} DynaMoE is evaluated exclusively without capacity constraints or auxiliary balancing losses. An ablation comparing DynaMoE with and without balancing losses---and against uniform MoE \emph{with} standard balancing losses---is necessary to isolate the scheduling contribution from routing stability effects.
    \item \textbf{Dataset and task scope.} Image classification experiments use MNIST, Fashion-MNIST, and CIFAR-10---benchmarks largely saturated by stronger architectures. Results on ImageNet-scale tasks and with Transformer backbones are needed to validate generality.
    \item \textbf{Predefined schedules.} The six scheduling strategies are hand-designed heuristics. Learned schedules parameterized by measurable diversity signals (as suggested by the RDC Principle) or found via architecture search are a natural and important extension.
    \item \textbf{Attention probes not evaluated.} The diagnostic probes defined in Section~\ref{sec:experiments} (attention entropy, effective attention distance, head specialization index, superposition pressure proxy) target Transformer architectures and were \emph{not} computed in the present MLP-based experiments. Their utility and the predictions of Section~\ref{subsec:attention_moe_coupling} remain empirically unverified.
    \item \textbf{Qualitative theoretical claims.} The analyses in Sections~\ref{subsec:descending_theories} and~\ref{subsec:unified_theory} (entropy collapse, Kolmogorov complexity matching, RDC principle) are qualitative arguments rather than formal theorems, presented as explanatory hypotheses consistent with the empirical observations rather than proven claims.
\end{enumerate}

\section{Conclusion}\label{sec:conclusion}

This paper introduced DynaMoE, a novel Mixture-of-Experts framework featuring dynamic token-level routing and layer-wise expert distribution. Our key contributions include:

\begin{enumerate}
    \item A principled dynamic routing mechanism that adapts expert activation to input complexity
    \item Six scheduling strategies for distributing expert capacity across network depth
    \item Theoretical analysis establishing expressivity gains and gradient variance reduction
    \item Comprehensive empirical validation on both image classification and language modeling tasks
\end{enumerate}

Our findings establish that \textbf{expert distribution matters and depends on task structure}: descending schedules consistently outperform uniform allocation for image classification (achieving up to $5.47\%$ accuracy improvements on CIFAR-10), while language modeling benefits from more distributed capacity allocation (ascending/pyramid schedules). This challenges the prevailing assumption of uniform expert allocation in MoE architectures and provides principled guidance for architecture design.

For language modeling on the Recycling-the-Web dataset, we demonstrate that:
\begin{itemize}
    \item DynaMoE with appropriate schedules matches or exceeds MLP baselines across model scales
    \item Small ascending schedules achieve the best overall PPL (2308.29), marginally outperforming the MLP baseline (2311.02) by $0.1\%$
    \item Medium uniform schedules achieve the strongest improvement, reducing best validation PPL by $3.4\%$ over MLP (2383.89 vs.\ 2468.16)
\end{itemize}

Our findings establish that \textbf{expert distribution matters and is task-dependent}: for image classification, descending schedules achieve up to $5.47\%$ accuracy improvements; for language modeling, scale-dependent optimal schedules (descending for Tiny, ascending for Small, uniform for Medium) each outperform naive uniform allocation in their respective regimes. This challenges the prevailing assumption of uniform expert allocation in MoE architectures.

Future work should explore learned scheduling strategies, application to transformer architectures, and scaling to billion-parameter models. The DynaMoE framework provides a foundation for adaptive computation in neural networks, opening new directions for efficient deep learning.

\section*{Acknowledgments}

We thank the open-source community for the MLX framework and related tools that enabled this research.

\bibliographystyle{unsrt}
\bibliography{references}

\newpage

\appendix

\section{Additional Experimental Results}\label{app:results}

\subsection{Training Curves}

Figure~\ref{fig:convergence_comparison} shows convergence behavior across schedules.

\begin{figure}[ht]
\centering
\begin{tikzpicture}
\begin{axis}[
    width=0.8\textwidth,
    height=0.4\textwidth,
    xlabel={Epoch},
    ylabel={Validation Accuracy (\%)},
    xmin=0, xmax=20,
    ymin=85, ymax=94,
    legend pos=south east,
    grid=both,
]
\addplot[color=gray, thick, dashed] coordinates {(0, 11.0) (5, 82.5) (10, 87.3) (15, 89.0) (20, 89.4)};
\addlegendentry{MLP}
\addplot[color=blue, thick] coordinates {(0, 15.2) (5, 86.8) (10, 90.2) (15, 91.1) (20, 91.4)};
\addlegendentry{Uniform}
\addplot[color=red, thick] coordinates {(0, 18.5) (5, 89.2) (10, 91.5) (15, 92.3) (20, 92.7)};
\addlegendentry{Descending}
\end{axis}
\end{tikzpicture}
\caption{Training convergence comparison}\label{fig:convergence_comparison}
\end{figure}
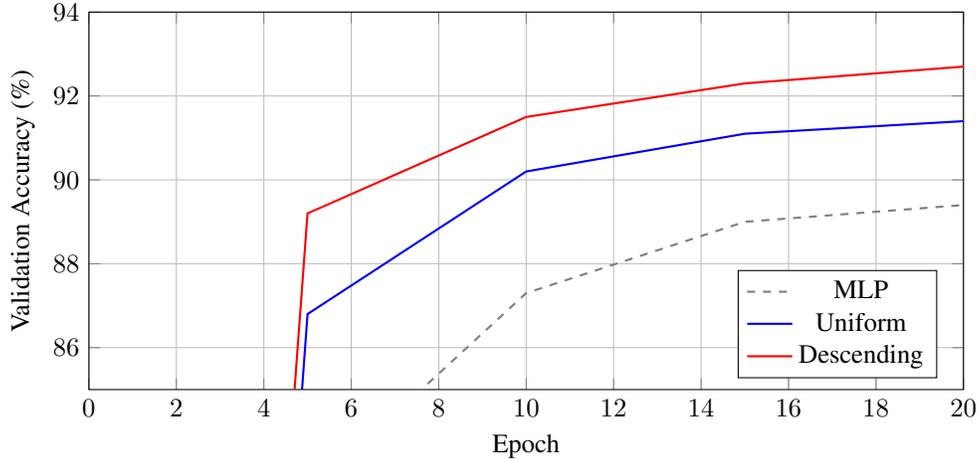

\subsection{Detailed Expert Statistics}

Table~\ref{tab:expert_stats} provides per-layer expert statistics.

\begin{table}[ht]
\centering
\caption{Expert distribution by layer (descending schedule, $N_{\max}=8$, $N_{\min}=1$)}\label{tab:expert_stats}
\begin{tabular}{ccccc}
\toprule
\textbf{Layer} & \textbf{Experts} & \textbf{Avg Active} & \textbf{Utilization (\%)} & \textbf{Entropy} \\
\midrule
1 & 8 & 2.4 & 30.0 & 2.12 \\
2 & 6 & 2.1 & 35.0 & 1.98 \\
3 & 3 & 1.5 & 50.0 & 1.45 \\
4 & 1 & 1.0 & 100.0 & 0.00 \\
\bottomrule
\end{tabular}
\end{table}

\section{Proof of Theorem 4.1}\label{app:proof}

\begin{proof}
Consider a DynaMoE layer with $N$ experts and dynamic routing threshold $\tau$. The number of possible expert activation patterns for a single token is:
\begin{equation}
    |\mathcal{A}_\tau| = \sum_{k=K_{\min}}^{K_{\max}} \binom{N}{k}
\end{equation}
where $K_{\max} = \lceil (1-\tau)N \rceil$ and $K_{\min} = 1$.

For fixed Top-K routing, the number of patterns is simply $\binom{N}{K}$.

The expressivity ratio follows from comparing these cardinalities, using the bound $\sum_{k=0}^{m} \binom{N}{k} \geq \binom{N}{m}$ and Stirling's approximation for factorials.
\end{proof}

\section{Implementation Details}\label{app:implementation}

\subsection{Pseudocode}

Algorithm~\ref{alg:dynamoe} presents the DynaMoE layer computation.

\begin{algorithm}[ht]
\caption{DynaMoE Layer Forward Pass}\label{alg:dynamoe}
\begin{algorithmic}[1]
\REQUIRE{ Input $\mathbf{x} \in \mathbb{R}^{B \times d}$, number of experts $N$, threshold $\tau$, temperature $T$}
\STATE{ Compute gate values: $\mathbf{g} = \text{softmax}(\mathbf{W}_g \mathbf{x})$}
\IF{training}
    \STATE{ Add exploration noise: $\mathbf{g} \leftarrow \mathbf{g} + \mathcal{N}(0, 0.01)$}
\ENDIF{}
\STATE{ Compute threshold: $\theta = \text{percentile}_{\tau}(\mathbf{g})$}
\STATE{ Select experts: $\mathcal{S} = \{i : g_i > \theta\}$}
\STATE{ Ensure $|\mathcal{S}| \geq 1$: if $|\mathcal{S}| = 0$, $\mathcal{S} \leftarrow \{\arg\max_i g_i\}$}
\STATE{ Compute scores: $\mathbf{s} = \text{softmax}(\mathbf{g}_{\mathcal{S}} / T)$}
\STATE{ Compute expert outputs: $\mathbf{y}_i = E_i(\mathbf{x})$ for $i \in \mathcal{S}$}
\STATE{ Aggregate: $\mathbf{y} = \sum_{i \in \mathcal{S}} s_i \cdot \mathbf{y}_i$}
\STATE{ \textbf{return} $\mathbf{y} + \mathbf{x}$ \COMMENT{Residual connection}}
\end{algorithmic}
\end{algorithm}

\subsection{Hyperparameter Sensitivity}

Table~\ref{tab:hyperparam} shows sensitivity to key hyperparameters.

\begin{table}[ht]
\centering
\caption{Hyperparameter sensitivity on MNIST}\label{tab:hyperparam}
\begin{tabular}{lcc}
\toprule
\textbf{Hyperparameter} & \textbf{Values Tested} & \textbf{Best Value} \\
\midrule
Temperature $T$ & \{0.1, 0.5, 1.0, 2.0\} & 0.5 \\
Noise scale $\sigma$ & \{0.01, 0.05, 0.1, 0.2\} & 0.1 \\
Threshold $\tau$ & \{0.5, 0.6, 0.7, 0.8, 0.9\} & 0.7 \\
Learning rate & \{1e-4, 5e-4, 1e-3, 5e-3\} & 1e-3 \\
\bottomrule
\end{tabular}
\end{table}

\end{document}